%% file: main.tex
\documentclass{article}

\usepackage{microtype}

\input{def.tex}

\usepackage{graphicx}
\usepackage{booktabs} 

\usepackage{hyperref}

\usepackage[accepted]{icml2024}

\usepackage{amsmath}
\usepackage{amssymb}
\usepackage{mathtools}
\usepackage{amsthm}

\usepackage{tikz}
\usetikzlibrary{calc}
\usepackage{pgfplots}
\usepackage{hyperref}
\usepackage{mathtools}
\usepgfplotslibrary{groupplots}
\usepgfplotslibrary{fillbetween}

\usepackage[capitalize,noabbrev]{cleveref}
\pgfdeclarelayer{front}
\pgfdeclarelayer{back}
\pgfsetlayers{back,main,front}

\theoremstyle{plain}
\newtheorem{theorem}{Theorem}[section]
\newtheorem{proposition}[theorem]{Proposition}
\newtheorem{lemma}[theorem]{Lemma}
\newtheorem{corollary}[theorem]{Corollary}
\theoremstyle{definition}
\newtheorem{definition}[theorem]{Definition}

\theoremstyle{remark}

\newcommand{\xvec}{\mathbf{x}}
\newcommand{\fvec}{\mathbf{f}}
\usepackage{subcaption}

\icmltitlerunning{Training Greedy Policy for Proposal Batch Selection in Expensive MOCO}

\begin{document}

\twocolumn[
\icmltitle{Training Greedy Policy for Proposal Batch Selection in \\Expensive Multi-Objective Combinatorial Optimization}



\icmlsetsymbol{equal}{*}

\begin{icmlauthorlist}
\icmlauthor{Deokjae Lee}{snu,nprc}
\icmlauthor{Hyun Oh Song}{snu,nprc}
\icmlauthor{Kyunghyun Cho}{nyu,gen}
\end{icmlauthorlist}

\icmlaffiliation{snu}{Seoul National University}
\icmlaffiliation{nprc}{Neural Processing Research Center}
\icmlaffiliation{nyu}{New York University}
\icmlaffiliation{gen}{Genentech}

\icmlcorrespondingauthor{Kyunghyun Cho}{kc119@nyu.edu}

\icmlkeywords{Machine Learning, ICML}

\vskip 0.3in
]



\printAffiliationsAndNotice{}  

\begin{abstract}
Active learning is increasingly adopted for expensive multi-objective combinatorial optimization problems, but it involves a challenging subset selection problem, 
optimizing the batch acquisition score that quantifies the goodness of a batch for evaluation.
Due to the excessively large search space of the subset selection problem,
prior methods optimize the batch acquisition on the latent space, which has discrepancies with the actual space, or optimize individual acquisition scores without considering the dependencies among candidates in a batch instead of directly optimizing the batch acquisition.
To manage the vast search space, a simple and effective approach is the greedy method, which decomposes the problem into smaller subproblems, yet it has difficulty in parallelization since each subproblem depends on the outcome from the previous ones.
To this end, we introduce a novel greedy-style subset selection algorithm that optimizes batch acquisition directly on the combinatorial space by sequential greedy sampling from the \emph{greedy policy}, specifically trained to address all greedy subproblems concurrently.
Notably, our experiments on the red fluorescent proteins design task show that our proposed method achieves the baseline performance in 1.69$\times$ fewer queries, demonstrating its efficiency.
\end{abstract}

\section{Introduction}
\label{introduction}
In various practical design fields, including biological sequence design, molecular graph optimization, and chip design, challenges are typically posed as expensive multi-objective combinatorial optimization (MOCO) problems. These problems focus on identifying designs, represented as discrete objects like strings or graphs, that optimize multiple attributes, often requiring substantial resources for accurate assessment \citep{multicriteria,lsbo,lambo,expensivemolecule,rlchip}.
Active learning frameworks, which iteratively propose a batch of candidates and learn from the attributes evaluated on those candidates, are increasingly employed in these fields due to their query efficiency, which is a critical component to handling expensive evaluation costs \citep{activelearning,bioseqgfn,lambo2,zhu2023sampleefficient,autodmp}.
In active learning, each round entails an internal problem of selecting a proposal batch of candidates for querying, formulated by cardinality-constrained subset selection problem.
This aims to identify the optimal batch $B\subset \mathcal{X}$ of size $n$ that maximizes the batch acquisition function $a:2^\mathcal{X}\to\reals$, which quantifies the goodness of a batch  considering interdependencies among candidates \citep{BatchBO,qei,qehvi}.
Unfortunately, the subset selection problem is challenging due to its prohibitively large search space of size $\mathcal{O}(|\mathcal{X}|^n)$, which increases exponentially as the batch size $n$ increases, while the combinatorial space $\mathcal{X}$ itself often has large size in practical scenarios \citep{moleculelarge}.

A natural approach to efficiently solve a subset selection problem is a greedy algorithm that sequentially constructs a subset by adding the optimal candidate that maximizes marginal gain in the objective set function, breaking down the problem into a sequence of substantially smaller, manageable subproblems of size $\mathcal{O}(|\mathcal{X}|)$ \citep{exactgreedybound}. 
Notably, the presence of monotone submodularity in prevalent batch acquisition functions such as JES, SM, EHVI, and NEHVI provides a theoretical performance guarantee for the greedy algorithm \citep{approxgreedy,sm}.
In this regard, active learning methods son continuous spaces already adopt greedy algorithms by solving each subproblem with first-order methods  \citep{jes,qehvi, qnehvi}.
However, for MOCO problems, applying greedy algorithms is even more challenging due to their discrete nature, which prohibits the use of first-order solvers.
Instead, previous works utilize latent space optimization (LSO) algorithms, which alternatively optimize the batch acquisition in the continuous latent space, which has discrepancies with the batch acquisition in the actual space, and obtain the batch by decoding the optimized latent 
 \citep{lsbo, lambo},
or construct a batch by sampling candidates of high individual acquisition scores $a(\{\mathbf{x}\})$s without considering the interdependencies among candidates explicitly \citep{jain2023multi}.

In this work, we propose a greedy-style subset selection algorithm for expensive MOCO problems.
One direct approach is sequentially applying any combinatorial optimization algorithm $n$ times to solve $n$ subproblems.
However, this sequential construction strategy has a drawback because each subproblem requires the results from preceding subproblems, hindering the parallelization of the overall process. 
To address this, we propose a novel subset selection approach based on reinforcement learning (RL) that trains only a single \emph{greedy policy}, a set-conditioned policy capable of addressing all subproblems concurrently, instead of sequentially training $n$ distinct policies for $n$ subproblems, thereby amortizing the burden of solving $n$ subproblems.
Our contributions can be summarized as follows:
\begin{itemize}\itemsep=2pt
    \item We propose a novel greedy-style subset selection algorithm that requires training of only a single greedy policy. Also, we suggest a novel training algorithm for obtaining the greedy policy, along with a justification for this approach.
    \item We extend the theoretical bounds of the approximated greedy algorithm to include both near-submodular functions and diversity functions, broadening its applicability. 
    \item Our method consistently outperforms baseline methods, constructing the batch with a higher batch acquisition in various benchmarks for active learning inner loops.  
    Significantly, our method attains the same Hypervolume indicator value as baseline methods but with 1.69$\times$ fewer queries in the multi-round active learning benchmark on red fluorescent proteins (RFP).
\end{itemize}

\begin{figure*}[ht]
    \centering
    \includegraphics[width=1.03\textwidth]{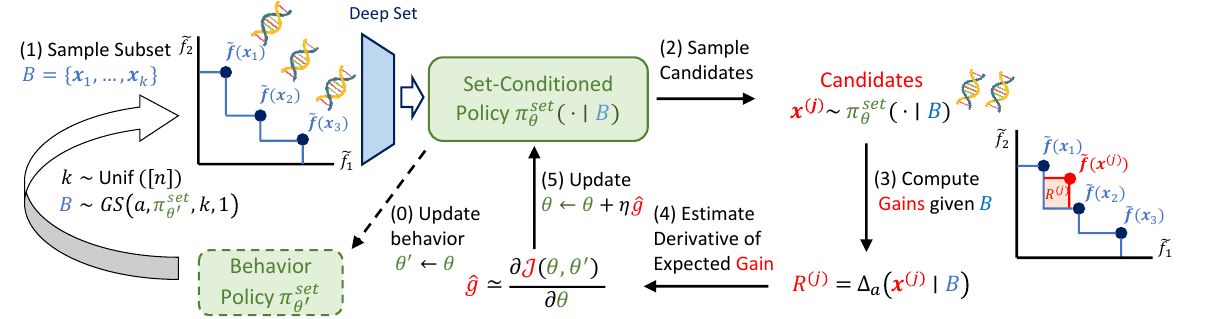}
    \caption{
    The visualization of our learning method (\Cref{sec:learning}). 
    At a high level, a set-conditioned policy $\pi_\theta^\text{set}$ is trained to generate candidates that maximize marginal gain $\Delta_a(\cdot\mid B)$ when conditioned by $B$, where $B$ is sampled by $\pi_\theta^\text{set}$ itself. 
    }
    \label{fig:first}
\end{figure*}

\section{Preliminaries}
\subsection{Expensive MOCO}
In this work, we consider an expensive MOCO problem, which aims to maximize an $m$-dimensional expensive, black-box oracle function $\mathbf{f}:\mathcal{X}\to\reals^m$ on the combinatorial space $\mathcal{X}$, \eg, the space of amino-acid sequences \citep{lsbo2}. 
This can be formulated as
\begin{equation}
    \maximize_{\xvec\in\mathcal{X}} ~\mathbf{f}(\xvec) \coloneqq (f_0(\xvec), \ldots, f_{m-1}(\xvec)),
    \label{eq:main}
\end{equation}
where we define the partial order between two vectors $\mathbf{f}(\xvec), \mathbf{f}(\xvec')\in\reals^m$ by the pointwise order, \ie, $\mathbf{f}(\xvec) \succeq \mathbf{f}(\xvec')$ if and only if $f_i(\xvec) \ge f_i(\xvec')$ for all $i\in[m]\coloneqq\{0,\ldots,m-1\}$ \citep{lattice}.
We say that a candidate $\xvec$ \emph{dominates} another candidate $\xvec'$, written by $\xvec \succ \xvec'$, if $\fvec(\xvec) \succeq \fvec(\xvec')$ and $\fvec(\xvec) \neq \fvec(\xvec')$. 
Using the notion of dominance, we introduce the set of the optimal solutions of \Cref{eq:main}, the \emph{Pareto set}, and its images, the \emph{Pareto frontier} \citep{moobasic}.
\begin{definition}(Pareto set and Pareto frontier)
\label{def:pareto}
    The Pareto set $\mathcal{P}^*\subset \mathcal{X}$ is the set of the optimal solutions that are not dominated by any other candidates in $\mathcal{X}$. Concretely,
    $\mathcal{P}^*:=\{\xvec\in\mathcal{X}\mid \nexists\xvec'\in\mathcal{X} \text{~s.t.~} \xvec'\succ \xvec\}$. Furthermore, the Pareto frontier is the images of the Pareto set under the oracle function $\mathbf{f}$, denoted as $\fvec(\mathcal{P}^*)\subset \reals^m$. 
\end{definition}
For the non-trivial scenarios, the Pareto set $\mathcal{P}^*$ has more than one solution due to trade-offs among oracle components $f_0, \ldots, f_{m-1}$ \citep{multicriteria}. 
Since we consider the scenario that a given black-box oracle function is expensive, the goal is to find a good approximated Pareto set $\tilde{\mathcal{P}}\subseteq\mathcal{X}$ in a limited query budget. 
To assess the quality of an approximation set $\tilde{\mathcal{P}}$, a commonly used metric is the \emph{Hypervolume indicator}, which measures the volume bounded by the reference point $\mathbf{r}_\text{ref}\in\reals^m$ and the images of the approximation set $\mathbf{f}(\tilde{\mathcal{P}})\subset \reals^m$, written by
$\mathbf{HV}(\mathbf{f}(\tilde{\mathcal{P}});\mathbf{r}_\text{ref}) \coloneqq \mathrm{Vol}(\bigcup_{\mathbf{y}\in \mathbf{f}(\tilde{\mathcal{P}})}\{\mathbf{v}\in\reals^m\mid \mathbf{r}_\text{ref}\preceq \mathbf{v} \preceq \mathbf{y}\})$ \citep{hvi}. 
In this work, we consider an approximation set with a higher Hypervolume indicator value as a better approximation of the Pareto set.
\subsection{Multi-Round Active Learning}
\begin{algorithm}[tb]
   \caption{Multi-Round Active Learning}
   \label{alg:al}
\begin{algorithmic}
   \STATE {\bfseries Input:} an oracle $\mathbf{f}$, a surrogate model $\tilde{\mathbf{f}}$, a batch size $n$, the number of rounds $N_r$, and the initial dataset $\mathcal{D}_0$.
   \FOR{$i=0$ {\bfseries to} $N_r-1$}
   \STATE Train a surrogate $\tilde{\mathbf{f}}$ using $\mathcal{D}_{i}$.
   \STATE Solve \Cref{eq:subsetselect} to select the batch $B$ (inner loop).
   \STATE Evaluate the batch $B$ with the oracle $\mathbf{f}$. 
   \STATE Update the dataset $\mathcal{D}_{i+1} \leftarrow  \mathcal{D}_{i} \cup \{(\xvec, \mathbf{f}(\xvec))\}_{\xvec\in B}$.
   \ENDFOR
\STATE \textbf{Return} $\mathrm{NonDominatedSort}(\mathcal{D}_{N_r})$
\end{algorithmic}
\end{algorithm}

\emph{Multi-round active learning} is a framework widely adopted for optimizing an expensive oracle function in a query-efficient manner \citep{activelearning}. 
This framework involves repeatedly suggesting candidates and learning from the oracle's feedback on those candidates \citep{bioseqgfn}. 
Due to the significant time costs of oracle queries, a common practice is to select a \emph{batch} of candidates for each round, enabling parallel evaluation by the oracle and thus improving overall efficiency \citep{qnehvi}. 
The \emph{batch Bayesian optimization} (BO) framework is a representative active learning approach equipped with a statistical surrogate model that estimates the oracle using the posterior distribution given previous oracle evaluations \citep{BatchBO}.
\Cref{alg:al} summarizes the overall active learning process. For each round, a cheaper surrogate model $\tilde{\mathbf{f}}(\cdot;\theta)$ is trained using data from previous steps to estimate the expensive oracle function $\mathbf{f}$.
Subsequently, the \emph{inner loop} chooses the proposal batch of $n$ candidates for querying by solving the following cardinality-constrained subset selection problem:
\begin{align}
\label{eq:subsetselect}
    &\maximize_{B\subset \mathcal{X}} ~a(B)\\
    &~\mathrm{subject~to} ~|B|\le n,\nonumber
\end{align}
where $a(\cdot;\tilde{f}):2^\mathcal{X}\to\reals$ is a \emph{batch acquisition function}, introduced in the subsequent section.
After finishing $N_r$ rounds, active learning returns non-dominated solutions among the evaluated dataset $\mathcal{D}_{N_r}$ as the approximation set to the Pareto set.

\subsection{Batch acquisition functions for MOCO}
Acquisition functions are designed to quantify the value of evaluating candidates throughout the active learning process, balancing the trade-offs between exploitation and exploration \citep{uncertainbellman}. 
Specifically, batch acquisition functions assess the value of evaluating a batch $B$, further considering the interdependencies within the batch \citep{qei,jes,qpoi,sm,botied}. 
In this work, we focus on batch acquisition functions based on \emph{Hypervolume improvement} (HVI), widely used in recent studies on expensive multi-objective optimization \citep{hvi,dgemo,lambo,jain2023multi,expensivepsl}.
Given the evaluated solution set $\tilde{\mathcal{P}}$ and the reference point $\mathbf{r}_\text{ref}$, HVI when evaluating a batch $B$ is defined as 
\begin{align*}
&\mathbf{HVI}(B;\tilde{\mathbf{f}},\tilde{\mathcal{P}}, \mathbf{r}_\text{ref})\\
&~~~~~~~\coloneqq 
\mathbf{HV}(\tilde{\mathbf{f}}(B)\cup\tilde{\mathbf{f}}(\tilde{\mathcal{P}});\mathbf{r}_\text{ref}) -
\mathbf{HV}(\tilde{\mathbf{f}}(\tilde{\mathcal{P}});\mathbf{r}_\text{ref}).
\end{align*}
HVI can be directly utilized as an acquisition function for a deterministic surrogate model. For a statistical surrogate model, variations of HVI such as EHVI, NEHVI, and UCB-HVI exist \citep{qehvi, qnehvi,ucbhvi}. 
EHVI and NEHVI compute the expected values of HVI under the assumptions of noiseless and noisy observations, respectively. 
UCB-HVI utilizes upper confidence bound (UCB) defined as $\tilde{\fvec}_\text{UCB}(\xvec;\beta)\coloneqq \mathrm{mean}(\tilde{\fvec}(\xvec)) + \beta ~ \mathrm{std}(\tilde{\fvec}(\xvec))\in\reals^m$ as a proxy vector of the oracle and computes $\mathbf{HVI}(B;\tilde{\fvec}_\text{UCB}, \tilde{\mathcal{P}}, \mathbf{r}_\text{ref})$.

\subsection{Reinforcement Learning for Single-Objective Combinatorial Optimization}
\label{main:rlintro}
Combinatorial objects can often be constructed by sequential actions. 
Deep RL-based methods for single-objective combinatorial optimization, aimed at solving $\max_{\xvec\in\mathcal{X}} R(\xvec)$, train a parameterized stochastic policy $\pi_\theta$. This policy, which determines actions at each state of object construction, seeks to maximize the objective function $R$ \citep{rlnas,rlchip}.
In this context, the resulting state of a trajectory $\tau$, sampled from a policy $\pi_\theta$, corresponds to a combinatorial object $\xvec\in\mathcal{X}$, and the return of the trajectory is given by $R(\xvec)$. For simplicity, we use the same notation for the trajectory and the resulting combinatorial object; therefore, for $\xvec\sim \pi_\theta$, $R(\xvec)$ represents the objective value of the corresponding combinatorial object $\xvec$, and $\pi_\theta(\xvec)$ represents the probability of the trajectory $\xvec$. Then, the given problem $\max_{\xvec\in\mathcal{X}} R(\xvec)$ can be translated to the following optimization problem:
$$\maximize_{\theta\in\Theta} ~\mathbb{E}_{\xvec\sim \pi_\theta}[R(\xvec)].$$
In this work, we consider the most basic algorithm with the REINFORCE update rule,
$\theta \leftarrow \theta + \eta R(\mathbf{x})\nabla_\theta\log\pi_\theta(\xvec)$,  as the optimization method \citep{reinforce}.
After training the policy, we can obtain the solution by sampling from the trained policy.

\section{Methods}
\begin{algorithm}[tb]
   \caption{Exact Greedy \citep{exactgreedybound}}
   \label{alg:exactgreedy}
\begin{algorithmic}
   \STATE {\bfseries Input:} a monotone set function $a:2^\mathcal{X}\to\reals$, and a cardinality constraint $n$.
   \STATE \textbf{Initialize} $B_0=\emptyset$.
   \FOR{$i=0$ {\bfseries to} $n-1$}
   \STATE $\xvec_i^* \leftarrow \argmax_{\xvec \in \mathcal{X}\setminus B_i } \Delta_a(\xvec\mid B_i)$.
   \STATE $B_{i+1} \leftarrow B_i \cup \{\xvec_i^*\}$.
   \ENDFOR
   \STATE \textbf{Return} $B_{n}$
\end{algorithmic}
\end{algorithm}

\label{sec:main}
The subset selection problem (\Cref{eq:subsetselect}) for each inner loop is
challenging due to its large search space of the size $\mathcal{O}(|\mathcal{X}|^{n})$, which increases exponentially as the cardinality constraint $n$ increases. 
To tackle these challenges, a simple yet effective approach is the greedy algorithm (\Cref{alg:exactgreedy}), which decomposes the subset selection problem into a series of smaller, more manageable subproblems, each of size $\mathcal{O}(|\mathcal{X}|)$ \citep{exactgreedybound}. 
Specifically, each greedy subproblem maximizes the \emph{marginal gain} $\Delta_a(\cdot\mid B):\mathcal{X}\to\reals$ of a set function $a:2^\mathcal{X}\to\mathbb{R}$, defined as $\Delta_a(\xvec\mid B) \coloneqq a(B\cup\{\xvec\}) - a(B)$  \citep{krause}.
However, the combinatorial space $\mathcal{X}$ often has a large size owing to its high-dimensional characteristics in practical applications, such as in biological sequence design \citep{lambo}.
Hence, finding the exact solution for each subproblem remains a formidable challenge. 
Addressing this, we focus on the approximated greedy algorithm (\Cref{alg:approxgreedy}), which utilizes a scalable maximization algorithm $\mathcal{A}$ such as sampling-based heuristics, genetic algorithms, or MDP-based methods like RL to approximate solutions within the large space of $\mathcal{X}$ \citep{approxgreedy,lazier,submodulardiverse,reinforce}. 

On the other hand, the sequential nature of the greedy-style algorithms potentially hinders efficiency, as each subproblem depends on the solutions of preceding steps, complicating the parallelization of the overall process.
To this end, we introduce a novel greedy-style subset selection method utilizing the greedy policy, a set-conditioned policy trained to handle all subproblems, with its novel training algorithm, thereby amortizing the burden of solving $n$ subproblems.
Furthermore, we elucidate the theoretical bound of the approximated greedy algorithm under various conditions of the set function, as encountered in realistic scenarios.

\begin{algorithm}[tb]
   \caption{Approx. Greedy \citep{approxgreedy}}
   \label{alg:approxgreedy}
\begin{algorithmic}
   \STATE {\bfseries Input:} a monotone set function $a:2^\mathcal{X}\to\reals$, a maximization algorithm $\mathcal{A}$, and a cardinality constraint $n$.
   \STATE \textbf{Initialize} $B_0=\emptyset$.
   \FOR{$i=0$ {\bfseries to} $n-1$}
   \STATE $\xvec_i \leftarrow$ solution by $\mathcal{A}$ when maximizing $\Delta_a(\cdot\mid B_i))$.
   \STATE $B_{i+1} \leftarrow B_i \cup \{ \xvec_i \}$.
   \ENDFOR
   \STATE \textbf{Return} $B_n$
\end{algorithmic}
\end{algorithm}
\begin{algorithm}[tb]
   \caption{Greedy Sample $\mathrm{GS}(a,\pi_\theta^\text{set},k,l)$}
   \label{alg:samplesetpolicy}
\begin{algorithmic}
   \STATE {\bfseries Input:} a monotone set function $a$, a set-conditioned policy $\pi_\theta^\text{set}$, a set size $k$, the number of samples $l$.
   \STATE \textbf{Initialize} $B_0=\emptyset$.
   \FOR{$i=0$ {\bfseries to} $k-1$}
   \STATE Sample $l$ candidates $\xvec_{i,0},\ldots,\xvec_{i,l-1}\sim \pi_\theta^\text{set}(\cdot\mid B_i)$. 
   \STATE $\text{idx} \leftarrow \argmax_{j\in[l]} \Delta_a(\xvec_{i,j}\mid B_i)$. 
   \STATE $B_{i+1} \leftarrow B_i \cup \{ \xvec_{\text{i,idx}} \}$.
   \ENDFOR
   \STATE \textbf{Return} $B_k$
\end{algorithmic}
\end{algorithm}
\subsection{Learning Greedy Policy}

\label{sec:learning}
To begin, we first introduce the concepts of a set-conditioned policy and a greedy sampling distribution. 
A \emph{set-conditioned policy} $\pi^\text{set}_\theta$ is a policy that samples a candidate $\xvec\sim\pi_\theta^{\text{set}}(\cdot\mid B)$ conditioned on any subset $B$. 
For any subproblem with a subset $B$, which maximizes $\Delta_a(\cdot\mid B)$, we may utilize $\pi^\text{set}_\theta(\cdot\mid B)$ as a proposal distribution to sample the solution.  
In this context, we define a \emph{greedy sampling} distribution $\mathrm{GS}(a,\pi_\theta^\text{set},k,l)$ on $k$-subsets of $\mathcal{X}$ as in \Cref{alg:samplesetpolicy}.

Instead of regarding greedy subproblems as individual problems to solve, we amortize all subproblems into a single problem of training a set-conditioned policy.
Specifically, the goal is to train a set-conditioned policy to be the \emph{greedy policy} $\pi_{\theta^*}^\text{set}$ that can exactly solve any subproblems encountered during the greedy sampling of $\pi_{\theta^*}^\text{set}$ itself.
To formally define the greedy policy, we first define the expected gain.
\begin{definition} (Expected Gain)
\label{def:expgain}
We define $\mathcal{J}(\theta,\theta')$ as the expected gain by $\pi^\text{set}_\theta$ given behavior policy $\pi^\text{set}_{\theta'}$ as:

\vspace{-1em}
{\small
    \begin{equation*} 
        \mathcal{J}(\theta,\theta')\coloneqq\frac{1}{n}\sum_{k=0}^{n-1}\mathbb{E}_{B\sim \mathrm{GS}(a,\pi_{\theta'}^\text{set},k,1)}[\mathbb{E}_{\xvec\sim\pi_\theta^\text{set}(\cdot\mid B)}[\Delta_a(\xvec\mid B)]].
    \end{equation*}}
\end{definition}
In short, the expected gain is the expected value of the marginal gain $\Delta_a(\xvec\mid B)$ where $\xvec \sim \pi_\theta^\text{set}(\cdot\mid B)$ and $B$ is any subset encountered during the greedy sampling of $\pi_{\theta'}^\text{set}$.
Using \Cref{def:expgain}, we formally define the greedy policy.
\begin{definition} (Greedy Policy)
\label{def:greedypolicy}
    A set conditioned policy $\pi_{\theta^*}^\text{set}$ is the greedy policy if $\pi_{\theta^*}^\text{set}$ is a maximizer of the expected gain given itself, \ie,
$\theta^* = \argmax_{\theta\in\Theta}\mathcal{J}(\theta,\theta^*).$
\end{definition}
To explain that the greedy policy defined in \Cref{def:greedypolicy} satisfies the desired property, we introduce \Cref{lem:greedypolicy}.
\begin{lemma}
\label{lem:greedypolicy}
$\mathrm{GS}(a,\pi_{\theta^*}^\text{set},n,l)$ samples exact greedy solutions almost surely if $\pi_{\theta^*}^\text{set}$ is the greedy policy.
\end{lemma}
Hence, the greedy policy is able to address all greedy subproblems and replicate the exact greedy algorithm.

From now on, we introduce the training method to achieve the greedy policy. To start, we define the partial derivative step, a fundamental component of our update rules.
\begin{definition}
\label{def:partial}
    Let $F:\mathcal{U}\times\mathcal{U}\to\reals$ be any continuously differentiable function. We define the \emph{partial derivative step} on $u\in\mathcal{U}$ given any behavior  $u'\in\mathcal{U}$ as $\mathrm{PD}_F(u;u',\eta) \coloneqq u + \eta \frac{\partial}{\partial u} F(u,u')$.
\end{definition}
Using the partial derivative step defined in \Cref{def:partial}, we introduce two update rules along with their validity. 
Similar to the assumptions such as strong convexity or smoothness used to ensure the convergence of the gradient descent algorithm \citep{bubeck}, we assume several conditions (see \Cref{sec:convergeproof}) and demonstrate the convergence of the update rules to the greedy policy under these assumptions.
\begin{theorem}
\label{thm:upd}
    Let $F:\mathcal{U}\times \mathcal{U}\to\reals$ be a function with some nice conditions. Also, assume that there exists $u^*\in\mathcal{U}$ such that $u^*=\argmax_{u\in\mathcal{U}}F(u,u^*)$. 
    Then, by iterating the update rule $u\leftarrow \mathrm{PD}_F(u;u'=u,\eta)$, $u$ converges to $u^*$ for a small $\eta>0$. Furthermore, for any  $N_t>0$, by iterating the update rule with $N_t$ partial derivative steps with a fixed behavior, \ie, $u\leftarrow \left(\mathrm{PD}_F(\cdot ; u'=u,\eta)\right)^{N_t}\!(u)$, $u$ converges to $u^*$ for a small $\eta>0$.
\end{theorem}
\begin{proof}
    Please refer to \Cref{sec:convergeproof} for the detailed statements and proofs.
\end{proof}
For clarity, the update rule with $N_t$ partial derivative steps in \Cref{thm:upd} on $u$ returns $u_n$ where
$$u_{i+1} = \mathrm{PD}_F(u_i ; u'=u,\eta)$$
for $i=0, \ldots, n-1$, and $u_0=u$.
Note that in practical scenarios, $\mathcal{J}$ may not fulfill these assumptions. However, our method still demonstrates empirical effectiveness in the practical scenarios, as shown in  \Cref{sec:results}. 

Due to the infeasible expectation in $\mathcal{J}$, we apply the update rule with an MC estimator $\hat{g}$ of $\frac{\partial}{\partial \theta} \mathcal{J}(\theta,\theta')$ which can be computed as in \Cref{prop:setgrad}.
\begin{proposition}
\label{prop:setgrad} (Policy gradient for $\mathcal{J}$)
For any baseline set function $b:2^\mathcal{X}\to\reals$ and the number of episodes $N_e$,

\vspace{-1.5em}
{\small
$$\hat{g}=\frac{1}{N_e}\sum_{j=0}^{N_e-1}(\Delta_a(\xvec^{(j)}\mid B) - b(B)) \nabla_\theta \log\pi^\text{set}_\theta(\xvec^{(j)}\mid B), $$}

\vspace{-1em}
is an unbiased MC estimator of  the partial derivative $\frac{\partial}{\partial \theta} \mathcal{J}(\theta,\theta')$ where $B\sim \frac{1}{n}\sum_{k=0}^{n-1}\mathrm{GS}(a,\pi_{\theta'}^\text{set},k,1)$ and $\xvec^{(j)}\sim \pi_\theta^\text{set}(\cdot\mid B)$ for all $j\in [N_e]$.
\end{proposition}
\Cref{alg:trainsetpolicy} outlines the overall process of our training algorithm.
Since every $N_t$ step shares the behavior policy, we can further speed-up the subset sampling by concurrently sampling $N_t$ subsets parallelly.
For the stable training, we use baseline techniques to normalize the returns. 
Finally, \Cref{fig:first} summarizes our proposed learning method.

\begin{algorithm}[tb]
   \caption{Training Set-Conditioned Policy}
   \label{alg:trainsetpolicy}
\begin{algorithmic}
   \STATE {\bfseries Input:} a monotone set function $a:2^\mathcal{X}\to \reals$, a cardinality constraint $n$, the number of updates $N_u$, the number of episodes per update $N_e$, a learning rate $\eta$, and a period of the behavior policy update $N_t$.
   \STATE \textbf{Initialize} $\theta$ randomly.
   \FOR{$i=0$ {\bfseries to} $N_u-1$}
   \IF{$i~\%~N_t = 0$}
    \STATE Update the behavior policy $\pi_{\theta'}^\text{set}\leftarrow \pi_\theta^\text{set}$.
   \ENDIF
   \STATE Sample a set size $k \sim \mathrm{Unif}([n])$.
   \STATE Sample a subset $B \sim \mathrm{GS}(a, \pi_{\theta'}^\text{set}, k, 1)$.
   \STATE Sample $N_e$ candidates $\xvec^{(1)}, \ldots, \xvec^{(N_e)}\sim \pi_\theta^\text{set}(\cdot\mid B)$.
   \FOR{$j=0$ {\bfseries to} $N_e-1$}
       \STATE Compute returns $r_j\leftarrow \Delta_a(\xvec^{(j)}\mid B)$.
       \STATE Normalize $\hat{r}_j \leftarrow (r_j - \mathrm{mean}(r))/(\mathrm{std}(r)+\epsilon)$.
   \ENDFOR
   \STATE Update $\theta \leftarrow \theta + \eta \sum_{j=0}^{N_e-1} \nabla_\theta [\hat{r}_j \log \pi_\theta^\text{set}(\xvec^{(j)}\mid B)]$.
   \ENDFOR
\end{algorithmic}
\end{algorithm}

\subsection{Architecture for set-conditioned policy} 
In this section, we explain the architecture of a set-conditioned policy.
Our architecture is inspired by preference-conditioned policies \citep{lin2022pareto,jain2023multi,zhu2023sampleefficient}.
The architecture of a preference-conditioned policy $\pi^\text{pc}_\theta$ can be summarized as follows:
$$\pi_\theta^\text{pc}(\mathbf{s} \mid \mathbf{v}) = \mathrm{Dec}(\mathrm{Enc}_\text{pref}(\mathbf{v})\oplus \mathrm{Enc}_\text{state}(\mathbf{s})),$$
where $\mathbf{v}$ is a preference vector, $\mathbf{s}$ is a given state, and the decoder $\mathrm{Dec}$ outputs a probability vector on the action space. 
Instead of the preference encoder $\mathrm{Enc}_\text{pref}$, we propose an architecture using a set encoder $\mathrm{Enc}_\text{set}$, \ie, 
\begin{equation}
\label{eq:setarch}
    \pi_\theta^\text{set}(\mathbf{s} \mid B) = \mathrm{Dec}(\mathrm{Enc}_\text{set}(B)\oplus \mathrm{Enc}_\text{state}(\mathbf{s})).\nonumber
\end{equation}
For the set encoder, we utilize the \emph{deep set} architecture, which is designed to encode a set of continuous vectors into a single embedding vector \citep{deepset}.
We extract $\mathrm{feat}(\xvec)$, the lower-dimensional continuous features to guide the policy, for each object $\xvec\in B$, and utilize the set of extracted features as input to the deep set encoder, \ie,
$\mathrm{Enc}_\text{set}(B)\coloneqq \mathrm{DeepSet}(\{\mathrm{feat}(\xvec)\mid \xvec\in B\}).$
It appears necessary to train an auxiliary feature extractor for the combinatorial space $\mathcal{X}$, but we already have a straightforward candidate for $\mathrm{feat}$, especially when using the deterministic surrogate model $\tilde{\fvec}$ with HVI, thanks to \Cref{lem:easy}.
\begin{lemma}
\label{lem:easy}
    For any $B,B'\subset\mathcal{X}$ satisfying $\tilde{f}(B)=\tilde{f}(B')$, $\Delta_a(\xvec\mid B) = \Delta_a(\xvec\mid B')$ for all $\xvec\in\mathcal{X}$.
\end{lemma}
\Cref{lem:easy} indicates that $\pi_\theta^\text{set}(\cdot\mid B)$ and $\pi_\theta^\text{set}(\cdot\mid B')$ solve identical problems if $\tilde{\fvec}(B)=\tilde{\fvec}(B')$. 
Hence, we simply set $\mathrm{feat}(\xvec)=\tilde{f}(\xvec)\in\reals^m$.
For cases using HVI-based batch acquisition functions with a statistical surrogate model $\tilde{f}$, we set $\mathrm{feat}(\xvec)=\tilde{f}_\text{UCB}(\xvec)\in\reals^m$. 

For the biological sequence design problems, we further utilize the MLM model, trained on previously evaluated data, into the decoder, inspired by the architecture proposed in LaMBO \citep{lambo}.
Please refer to \Cref{sec:arch} for more details on architectures.

Utilizing a learned set-conditioned policy $\pi_{\theta^*}^\text{set}$, we construct a proposal batch of $n$ candidates for querying by sampling from $\mathrm{GS}(a,\pi_{\theta^*}^\text{set},n,l)$. 
Please refer to \Cref{sec:lemsproof} for the proofs of \Cref{lem:greedypolicy}, \Cref{prop:setgrad}, and \Cref{lem:easy}

\subsection{Bounds for Approximated Greedy Algorithm}
\label{sec:bounds}
In this section, we introduce bounds for the approximated greedy algorithm under various conditions of the set function, as encountered in active learning scenarios.
First, we introduce the concept of an $\alpha$-approximation algorithm which indicates the amount of approximation as in \Cref{def:alphaapp}.
\begin{definition} ($\alpha$-Approximation Algorithm)
\label{def:alphaapp}
    An algorithm $\mathcal{A}$ is $\alpha$-\emph{approximation algorithm} if $\xvec_i$ found by $\mathcal{A}$ is an $\alpha$-approximation to the exact solution $\xvec_i^*$ for each step in \Cref{alg:approxgreedy}, \ie, $\Delta_a(\xvec_i\mid B_i)\ge \alpha \Delta_a(\xvec_i^*\mid B_i)=\alpha \max_{\xvec\in\mathcal{X}\setminus B_i}\Delta(\xvec\mid B_i)$.
\end{definition}
Prior research established bounds for approximated greedy algorithms in the term of $\alpha$ for submodular batch acquisition functions \citep{approxgreedy}. 
Common batch acquisition functions, such as EHVI and PES, are submodular, but their exact values are infeasible to compute due to the expectation involved \citep{qehvi,pes}. 
Instead, these values are estimated using methods like MC sampling or expectation propagation. 
Hence, the realized values of these acquisition functions can be near-submodular rather than strictly submodular.

Inspired by \citet{submodindexbound}, we propose a bound for the approximated greedy algorithm when the batch acquisition function $a$ is near-submodular, using the submodularity ratio $\gamma$ which quantifies the degree of submodularity in $a$.
\begin{theorem}
\label{thm:modbound}
    Let $a:2^\mathcal{X}\to\reals$ be a non-negative monotone set function. If $\mathcal{A}$ is an $\alpha$-approximation algorithm, the resulting solution $B_n$ of \Cref{alg:approxgreedy} is an $(1-1/e^{\alpha\gamma_{B_n,n}(a)})$-approximation to the optimal $n$-subset $B_n^*$, \ie, $a(B_n)\ge (1-1/e^{\alpha\gamma_{B_n,n}(a)}) a(B_n^*)$.
\end{theorem}

Furthermore, there are works adopting heuristics to enhance the input diversity among evaluated candidates, aiming to identify diverse modes in the search space \citep{dgemo,bioseqgfn,jain2023multi,zhu2023sampleefficient}.
The batch acquisition function for this \emph{diversified subset selection problem} can be expressed as $a(B)\coloneqq s(B) + \lambda \mathrm{div}(B_\text{prev}\cup B)$, where $\lambda$ is a coefficient controlling the tradeoff, $s$ is a generic batch acquisition function, $\mathrm{div}$ quantifies the input diversity, and $B_\text{prev}\subset \mathcal{X}$ is the set of previously evaluated candidates.
We mainly consider the diversity in the form of the \emph{sum-dispersion}, defined as $\mathrm{div}(U)\coloneqq 1/2~d(U,U) = 1/2~\sum_{\xvec \in U} \sum_{\xvec'\in U} d(\xvec,\xvec')$ for any metric $d$, due to its flexibility \citep{dispersion,borodin}. In this case, we can simplify the batch acquisition function $a$ as follows:

\vspace{-0.8em}
{\small
\begin{align*}
    a(B) = s(B)+ \underbrace{\underbrace{\lambda \mathrm{div}(B_\text{prev})}_{\text{constant on $B$}}+ \lambda \sum_{\xvec\in B} \underbrace{d(\xvec, B_\text{prev})}_{\text{unary on $\xvec$}} }_{=:\mathrm{aux}(B),~ \text{non-negative monotone modular}} + \lambda \mathrm{div}(B).
\end{align*}}

\vspace{-0.8em}
\noindent For simplicity, we assume that the batch acquisition function is given by $a(B)=s(B)+\lambda \mathrm{div}(B)$,  neglecting the non-negative monotone modular term $\mathrm{aux}(B)$, since $(s+\mathrm{aux})$ is non-negative monotone  for any non-negative monotone $s$ \citep{submodularbook}.
Now, we propose a bound for the approximated greedy algorithm for the diversified subset selection problem, extending the bound for submodular cases proved in \citet{borodin}.
\begin{theorem}
\label{thm:divmodbound}
    Let $s$ be a non-negative monotone set function and $\mathrm{div}$ be a sum-dispersion. If $\mathcal{A}$ is an $\alpha$-approximation algorithm, 
    \Cref{alg:approxgreedy} with the set function $(s/2+ \lambda \mathrm{div})$ returns $B_n$, an $(\alpha\hat{\gamma}/2)$-approximation to the optimal $n$-subset $B_n^*$ of $(s+\lambda \mathrm{div})$ where $\hat{\gamma}\coloneqq \gamma_{B_n\cup B_n^*,n}(s)$.
\end{theorem}

\Cref{thm:modbound} and \Cref{thm:divmodbound} demonstrate that even if our learning algorithm does not perfectly learn the greedy policy, the performance bound can still be guaranteed in terms of the level of approximation $\alpha$ in subproblems, even under these realistic conditions beyond submodularity.
Please refer to \Cref{app:boundsdetail} for the detailed statements and proofs of the theorems, as well as their connections to prior works.


\section{Experiments}
\label{sec:results}
\begin{table*}[hbt]
    \centering
    \caption{Subset selection results on three bigrams tasks with various cardinality constraint $n$. Each value indicates the Hypervolume indicator of discovered subset. The mean and standard deviation values are calculated for 10 trials.}
\resizebox{\textwidth}{!}{
\begin{tabular}{lccccccccccccccccccc}
\toprule
& \multicolumn{9}{c}{Hypervolume Indicator ($\uparrow$)}\\
\cmidrule(lr){2-10}
         & \multicolumn{2}{c}{2 Bigrams}& \multicolumn{3}{c}{3 Bigrams}& \multicolumn{4}{c}{4 Bigrams}\\
    \cmidrule(lr){2-3}
    \cmidrule(lr){4-6}
    \cmidrule(lr){7-10}
Method     &  $n=4$ & $n=16$ &  $n=4$ & $n=16$ & $n=64$& $n=4$ & $n=16$ & $n=64$ & $n=256$\\
    \cmidrule(lr){1-3}
    \cmidrule(lr){4-6}
    \cmidrule(lr){7-10}
Optimum & \multicolumn{2}{c}{0.630} & \multicolumn{3}{c}{0.409}  & \multicolumn{4}{c}{0.106} \\
    \cmidrule(lr){1-3}
    \cmidrule(lr){4-6}
    \cmidrule(lr){7-10}
Exact Greedy& 0.568 & 0.630 & 0.350 & 0.408 & 0.409 & 0.055 & 0.078 & 0.097 & 0.106   \\
    \cmidrule(lr){1-3}
    \cmidrule(lr){4-6}
    \cmidrule(lr){7-10}
Ours        & \textbf{0.568} (0.000)& \textbf{0.630} (0.000) & \textbf{0.329} (0.005) & \textbf{0.349} (0.007) & \textbf{0.359} (0.003)&\textbf{0.055} (0.000) &\textbf{0.077} (0.000)& \textbf{0.091} (0.002)& \textbf{0.094} (0.003) \\
PC-RL (TS)      &0.558 (0.007)&0.620 (0.003)&0.318 (0.012)&0.347 (0.003)&\textbf{0.359} (0.004)&0.040 (0.005) & 0.054 (0.003) & 0.071 (0.002) & 0.082 (0.007)\\
PC-RL (WS)      &0.522 (0.030)&0.583 (0.018)&0.310 (0.007)&0.337 (0.008)&0.347 (0.007)&  0.009 (0.009)& 0.016 (0.005)& 0.028 (0.005) & 0.032 (0.006)\\
Greedy + RL  & 0.518 (0.002) & 0.518 (0.040) & 0.320 (0.000)& 0.320 (0.008) & 0.322 (0.001) & 0.047 (0.004)&0.062 (0.005)&0.065 (0.008)&0.070 (0.004)\\ 
Greedy + HC & 0.063 (0.028) & 0.063 (0.028) & 0.182 (0.021) & 0.182 (0.021) & 0.171 (0.038) &0.014 (0.008) &  0.014 (0.008) & 0.013 (0.008) &0.001 (0.001)\\
Greedy + RS &  0.025 (0.002) & 0.027 (0.002) & 0.002 (0.001) &  0.003 (0.000) &  0.003 (0.000)  &  0.000 (0.000) &  0.000 (0.000)&  0.000 (0.000) & 0.000 (0.000)\\
\bottomrule
\end{tabular}}
\label{tab:main}
\end{table*}
We validate the performance of our proposed method on various benchmarks on sequence design problems based on the tasks suggested by \citet{lambo} and \citet{jain2023multi}.
First, we outline the benchmarks and the baseline methods.
Next, we compare the subset selection performance of our method with the baseline methods on single-round synthetic tasks, which correspond to the deterministic surrogate model scenario. 
Finally, we present the results on batch BO scenarios, which utilize stochastic surrogate models.
Our implementation is available at \href{https://github.com/snu-mllab/GreedyPolicyForMOCO}{https://github.com/snu-mllab/GreedyPolicyForMOCO}.

\subsection{Settings}
\textbf{Single-round experiments on synthetic tasks.}
In this setting, we assume that a deterministic surrogate model is given as a synthetic function.
We mainly consider the subset selection problem that maximizes the Hypervolume indicator value of given deterministic synthetic functions on bigram matching tasks with various numbers of objectives and a DNA aptamer design task with three objectives computed by NUPACK library \citep{jain2023multi, nupack}. 
We compare our method with preference-conditioned RL methods with Chebyshev scalarization, denoted PC-RL (TS), and weight scalarization, denoted PC-RL (WS) \citep{lin2022pareto}. 
We also compare our method with the approximated greedy algorithm augmented with combinatorial algorithms: RL (Greedy + RL), hill climbing (Greedy + HC), and random sampling (Greedy + RS) \citep{reinforce, selman2006hill, lazier}. 
For a fair comparison, we fix the number of queries to the deterministic surrogate model across all methods.

\begin{table}[t]
    \centering
    \caption{Subset selection results on the DNA aptamer design task with various cardinality constraint $n$. The mean and standard deviation values are calculated for 10 trials.}
    \resizebox{\columnwidth}{!}{
    \begin{tabular}{lccccc}
    \toprule
    &\multicolumn{4}{c}{Hypervolume Indicator ($\uparrow$)}\\
    \midrule
         Method& $n=4$& $n=16$ & $n=64$ & $n=256$ \\
         \midrule
         Ours& \textbf{0.662} (0.019)& \textbf{0.717} (0.025) & \textbf{0.764} (0.045) & \textbf{0.778} (0.026)\\
         PC-RL (TS)& 0.515 (0.041)& 0.658 (0.060) &0.712 (0.052) & 0.731 (0.042) \\ 
         PC-RL (WS)& 0.479 (0.035)& 0.530 (0.009) &0.587 (0.027) & 0.611 (0.017)\\
         Greedy + RL & 0.551 (0.029) & 0.705 (0.027) & 0.749 (0.034) & 0.739 (0.023)\\
         Greedy + HC & 0.367 (0.039) &  0.388 (0.047) & 0.377 (0.050)& 0.372 (0.042)\\
         Greedy + RS & 0.199 (0.012) &  0.226 (0.012) &0.231 (0.015) &  0.232 (0.015) \\ 
         \bottomrule
    \end{tabular}}
    \label{tab:nupack}
\end{table}

\textbf{Batch BO experiments.}
In this scenario, we adopt the benchmark tasks proposed by \citet{lambo} to evaluate the  performance of our method with statistical surrogate models. 
We compare our method with several active learning methods: 
 LaMBO, an LSO method; MBGA, a model-based genetic algorithm; and AL-MOGFN, a GFlowNet-based active learning method that maximizes individual acquisition values while enhancing the diversity \citep{lambo,jain2023multi}.
For these methods, we use the same architecture and training algorithm for the statistical surrogate model \citep{mtgp}. 
As in \citet{lambo}, we utilize NEHVI as the batch acquisition  for all methods.
We also consider NSGA-II as a baseline method to show the performance of a model-free method that is not an active learning approach \citep{nsga2}.

Though active learning frameworks assume that the surrogate model is much cheaper than the oracle, we set  the number of surrogate model queries of our method to be comparable to the baseline methods for a fair comparison.
Please refer to \Cref{sec:settings} for more details on the experimental settings containing tasks and baselines.

\subsection{Results on Synthetic Tasks}
\Cref{tab:main} summarizes the single-round subset selection results on bigrams tasks. 
The results show that our method consistently outperforms the baseline methods, searching batches with higher Hypervolume indicator values compared to the baseline methods for all bigram tasks with various cardinality constraint values we consider.
In these tasks, we can obtain the optimum and Hypervolume indicator values of exact greedy solutions by rule-based backtracking search.
Notably, in the 2 bigrams task, our proposed algorithm finds subsets with Hypervolume indicator values that are the same as the exact greedy solutions. 
\Cref{tab:nupack} summarizes the subset selection results on the DNA aptamer design task. 
The results state that our method also outperforms the baseline methods in the DNA aptamer task, demonstrating the wide applicability of our method.
Please refer to \Cref{sec:addexp1} for the additional results on single-round synthetic tasks.

\subsection{Results on Batch BO Scenarios}
First, we note that the prior implementations of LaMBO and MBGA had several issues, such as incorrect computation of  NEHVI. 
We fix these issues and reproduce the results of baseline methods with more update steps for a fair comparison with our method. 
For detailed information, please see \Cref{sec:issues}. 
\Cref{fig:RFP} presents the experimental results on the RFP task. 
As shown in \Cref{fig:multiround}, both variations of our methods outperform the baseline methods in the RFP task.
Remarkably, our method with MLM achieves 25\% higher relative Hypervolume indicator value compared to the best baseline results from MBGA.  
In a different view, our method with MLM demonstrates comparable performance while requiring 1.69 times fewer queries.
\Cref{fig:frontier} illustrates the frontiers discovered by our method and AL-MOGFN. 
The frontiers obtained through our method completely dominate those previously discovered by AL-MOGFN, demonstrating the effectiveness of our proposed active learning method.
Additionally, we provide visualizations of the non-dominated offsprings for each color, highlighting that our method successfully discovers improved offsprings for all ancestor proteins.

To directly validate the subset selection performance in batch BO scenarios, we evaluate the resulting batch acquisition value of each method in the first round with the same surrogate model and the initial data. 
\Cref{tab:statnehvi} demonstrates that our method achieves higher batch acquisition values compared to baseline active learning methods, indicating the superiority of our method in inner loop optimization with a stochastic surrogate model. 
For additional results on batch BO experiments, please refer to \Cref{sec:addexp2}.

\begin{table}[t]
    \centering
    \caption{Subset selection results in the first round of the batch BO benchmarks (RFP, 3 Bigrams, small molecules) with NEHVI. The mean and standard deviation values are calculated for 10 trials.}
    \resizebox{\columnwidth}{!}{
    \begin{tabular}{ccccc}
    \toprule
    &\multicolumn{3}{c}{ NEHVI Value ($\uparrow$)}\\
    \cmidrule(lr){2-4}
Method & RFP & 3 Bigrams & Molecules\\
\midrule
Ours&  \textbf{0.779} (0.045) & \textbf{16.444} (1.529) & \textbf{0.698} (0.104)\\ 
LaMBO& 0.591 (0.033) & 9.922 (0.688) &	0.545 (0.064)\\
MBGA&0.654 (0.052) & 12.376 (1.576) & 	0.483 (0.124)\\ 
\bottomrule
    \end{tabular}}
    \label{tab:statnehvi}
\end{table}
\input{figures/fig_main}

\section{Related Works}
\textbf{Genetic algorithms (GAs) for MOCO} A widely-used off-the-shelf GA method for multi-objective optimization, NSGA-II, employs random mutations and a non-dominated sorting for iterative population updates \citep{nsga2, moobasic}.
\citet{eamoo} improved GA to better handle large combinatorial search spaces by utilizing graph neural networks.
While GAs are recognized for their simplicity and adaptability across various types of problems, generic GA methods may not be suitable for expensive MOCO problems because they often necessitate a large number of queries \citep{gabad}.

\textbf{MOCO based on Markov decision processes (MDPs).} An emerging research direction treats combinatorial optimization problems as decision-making problems, framing them within the context of MDPs and training policies via RL or GFlowNet frameworks \citep{rlnas,rlneuralco,rlchip,bengio2023gflownet}. 
In MOCO research, this perspective has led to the development of methods that train policies to generate the Pareto set \citep{decisionmaking,yang2019morl,nagentspsl1,nagentspsl2,lin2022pareto,jain2023multi,zhu2023sampleefficient}. 
These methods predominantly utilize preference-based scalarization techniques, such as weighted scalarization or weighted Chebyshev scalarization, to decompose multi-objective problems into single-objective subproblems \citep{chebyshev}. 
\citet{nagentspsl1} and \citet{nagentspsl2} train multiple RL agents, each corresponding to a single subproblem.
\citet{yang2019morl} amortize subproblems into a single problem with a stochastic reward scalarized by random preference vectors, demonstrating the training of a single RL agent to solve MOCO problems.
\citet{lin2022pareto} introduce a preference-conditioned RL agent, which is trained to maximize a scalarized reward corresponding to the conditioned preference vector.
In a different approach, \citet{jain2023multi} and \citet{zhu2023sampleefficient} employ GFlowNet frameworks to enhance the diversity of the learned solutions, applying this technique to solve expensive MOCO problems, such as biological sequence design and molecular graph discovery, in active learning frameworks.

\textbf{Latent space optimization (LSO) for MOCO.} Rather than directly searching for candidates in the explicit combinatorial space, \citet{lambo} proposed LaMBO, a LSO-style inner loop optimization method designed for discrete sequence data. This method involves training an autoencoder and proposing a batch of candidates through first-order optimization of the batch acquisition function in a continuous latent space \citep{lsbo, lsbo2, mtgp}.

\section{Conclusion}
We propose a query-efficient optimization method for expensive MOCO problems based on active learning utilizing a novel greedy-style subset selection algorithm.
In contrast to prior methods, our subset selection method explicitly optimizes the batch acquisition function on the combinatorial space.
Moreover, our subset selection algorithm trains a greedy policy to address all greedy subproblems simultaneously, overcoming the typical difficulty of parallelization in greedy-style approaches.
By utilizing the trained greedy policy, our algorithm constructs the proposal batch by sequential sampling from the greedy policy.
Furthermore, we extend the theoretical bound on approximated greedy algorithms for various types of set functions containing monotone set functions and sum-dispersion functions.
Empirical results on single-round subset selection benchmark in biological sequence designs show that our method consistently outperforms baseline methods, finding the batch with higher batch acquisition value.
Surprisingly, in multi-round active learning for red fluorescent protein design, our approach achieves the same level of performance as baseline methods but with 1.69$\times$ fewer queries, demonstrating its effectiveness and efficiency.

\section*{Impact Statement}
The primary objective of our research is to improve the end-to-end process for solving expensive MOCO problems.
Expensive MOCO problems encompass a wide range of complex challenges in society, including drug discovery and chip design, which have widespread impact. 
Through our approach of generating superior proposal candidates using a greedy policy, 
our work has the potential to contribute to the quicker development of new medications, and the creation of more efficient electronic devices, potentially offering benefits to both the industry and society.

\section*{Acknowledgement}
The work was done as part of the Meta–NYU mentorship program and partly supported by the National Science Foundation (under NSF Award 1922658). 
Kyunghyun Cho is supported by the Samsung Advanced Institute of Technology (under the project Next Generation Deep Learning: From Pattern Recognition to AI). 
Deokjae Lee and Hyun Oh Song are supported by Samsung Advanced Institute of Technology, 
Samsung Electronics Co., Ltd. (IO220810-01900-01), 
Institute of Information \& Communications Technology Planning \& Evaluation (IITP) grant funded by the Korea government (MSIT) [No. RS-2020-II200882, (SW STAR LAB) Development of deployable learning intelligence via self-sustainable and trustworthy machine learning and No. RS-2021-II211343, Artificial Intelligence Graduate School Program (Seoul National University)],
the National Research Foundation of Korea (NRF) grant funded by the Korea government (MSIT) (No. RS-2024-00354036), 
and Basic Science Research Program through the National Research Foundation of Korea (NRF) funded by the Ministry of Education (RS-2023-00274280). Kyunghyun Cho is the corresponding author.

\bibliography{example_paper}
\bibliographystyle{icml2024}

\newpage
\appendix
\onecolumn
\section{Mathematical Details}
\subsection{Proof of \Cref{thm:upd}}
\label{sec:convergeproof}
To start, we recall \Cref{def:partial} and \Cref{thm:upd}. 
\begin{definition}
\label{def:partialapp}
    Let $F:\mathcal{U}\times\mathcal{U}\to\reals$ be any continuously differentiable function. We define the \emph{partial derivative step} on $u\in\mathcal{U}$ given any behavior $u'\in\mathcal{U}$ as $\mathrm{PD}_F(u;u',\eta) \coloneqq u + \eta \frac{\partial}{\partial u} F(u,u')$.
\end{definition}
\begin{theorem}
\label{thm:updapp}
    Let $F:\mathcal{U}\times \mathcal{U}\to\reals$ be a function with nice conditions. Also, assume that there exists $u^*\in\mathcal{U}$ such that $u^*=\argmax_{u\in\mathcal{U}}F(u,u^*)$. 
    Then, by iterating the update rule $u\leftarrow \mathrm{PD}_F(u;u'=u,\eta)$, $u$ converges to $u^*$ for a small $\eta>0$. Moreover, for any  $N_t>0$, by iterating the update rule with $N_t$ partial derivative steps with a fixed behavior, \ie, $u\leftarrow \left(\mathrm{PD}_F(\cdot ; u'=u,\eta)\right)^{N_t}\!(u)$, $u$ converges to $u^*$ for a small $\eta>0$.
\end{theorem}
We first introduce the intuition of the proposed update rules. 
Assume that $F:\mathcal{U}\times \mathcal{U}\to\reals $ is a smooth function and $F_{u'}\coloneqq F(\cdot, u'):\mathcal{U}\to\reals$ is a strongly concave function on $\mathcal{U}$ for any $u'\in\mathcal{U}$.
Under this assumption, if $u,u'\in\mathcal{U}$ satisfies 
\begin{equation}
    \left.\frac{\partial}{\partial u} F(u, u')\right|_{u'=u} = 0,
    \label{eq:cond0}
\end{equation}
we have the following equality:
$$(\nabla F_{u})(u) = \left.\frac{\partial}{\partial u} F(u, u')\right|_{u'=u} = 0.$$
Due to the strong concavity, $F_u$ has a unique maximizer and we finally get the desired property as following:
$$u=\argmax_{v\in\mathcal{U}} F_u(v) = \argmax_{v\in\mathcal{U}}F(v,u).$$
Inspired by \citet{fixedpoint}, we design two update rules, whose fixed point satisfies \Cref{eq:cond0}, based on partial derivative steps in \Cref{def:partialapp}.
For simplicity, we introduce notations $\nabla_1 F$, $\nabla_2F$, $\nabla_1^2 F$, and $\nabla_2\nabla_1 F$ as 
$$(\nabla_1 F)(u,u') = \left(\frac{\partial}{\partial u} F\right)(u,u'), (\nabla_2 F)(u,u') = \left(\frac{\partial}{\partial u'} F\right)(u,u'),$$
 $$   (\nabla_1^2 F)(u,u') = \left(\frac{\partial^2}{\partial u^2} F\right)(u,u'), (\nabla_2\nabla_1 F)(u,u') = \left(\frac{\partial^2}{\partial u'\partial u}  F\right)(u,u').$$\\
From now on, we state the conditions on $F$ required in our proof for justifying the convergence of the update rules as follows.\\

\textbf{Conditions ($*$):}
\begin{enumerate}
    \item $F$ is a smooth function.
    \item $F_{u'}$ is a strongly concave function and the eigenvalues of the hessian $(\nabla^2 F_{u'})=\nabla_1^2 F(u,u')$ is uniformly bounded by two negative values $L\le M<0$, \ie, $$L \le \lambda_\text{min}\left(\nabla_1^2 F(u,u')\right) \le  \lambda_\text{max}\left(\nabla_1^2 F(u,u')\right) \le M < 0,$$
    for all $u, u'\in\mathcal{U}$.
    \item 
    The square matrix of second order derivatives $\nabla_2\nabla_1 F(u,u')$ has singular values bounded above by $M'<-M$, \ie,
    $$ 0 \le \sigma_\text{min}\left(\nabla_2\nabla_1 F(u,u')\right) \le  \sigma_\text{max}\left(\nabla_2\nabla_1 F(u,u')\right) \le M' < -M \le -L,$$
    for all $u,u'\in\mathcal{U}$.
\end{enumerate}
Assuming these three conditions ($*$) on $F$, we prove the following proposition first.
\begin{proposition}
\label{prop:contract}
If $F:\mathcal{U}\times\mathcal{U}\to\reals$ satisfies three conditions ($*$), there exists an $\eta >0$ and $0<\beta < 1$ such that the partial derivative step  $\mathrm{PD}_F(u;u',\eta)$ satiesfies the following property:
$$\|\mathrm{PD}_F(u_2;u_2',\eta) - \mathrm{PD}_F(u_1;u_1',\eta)\|_2 \le \beta\|u_2'-u_1'\|_2 $$
for all $u_1, u_1',u_2,u_2'\in\mathcal{U}$ such that $\|u_2-u_1\|\le \|u_2'-u_1'\|$. 
\end{proposition}
\begin{proof}
    \begin{align*}
        &\|\mathrm{PD}_F(u_2;u_2',\eta) - \mathrm{PD}_F(u_1;u_1',\eta)\|_2\\
        &= \left\|\left(u_2 + \eta \nabla_1F(u_2,u_2')\right) - 
        \left(u_1 + \eta \nabla_1F(u_1,u_1')\right)\right\|_2\\
        &= \left\| (u_2 - u_1) + \eta \left(\nabla_1 F(u_2,u_2') - \nabla_1 F(u_1,u_1')\right)
        \right\|_2\\
        &= \left\| (u_2 - u_1) + \eta \left(\nabla_1 F(u_2,u_2') - 
        \nabla_1 F(u_1, u_2') + \nabla_1 F(u_1, u_2') - 
        \nabla_1 F(u_1, u_1')\right)
        \right\|_2\\
        &= \left\| (u_2 - u_1) + \eta \left(\nabla_1 F(u_2,u_2') - 
        \nabla_1 F(u_1, u_2')\right) + \eta \left(\nabla_1 F(u_1, u_2') - 
        \nabla_1 F(u_1, u_1')\right)
        \right\|_2\\
        &= \left\|
        (u_2 - u_1) + \eta \int_{0}^1 \nabla_1^2 F(u_1 + t(u_2-u_1), u_2')(u_2-u_1)dt + \eta \int_{0}^1 \nabla_2\nabla_1 F(u_1, u_1' + t(u_2'-u_1'))(u_2'-u_1')dt
        \right\|_2\\
        &\le \Bigg\|
        \underbrace{\left(I + \eta \int_{0}^1 \nabla_1^2 F(u_1 + t(u_2-u_1), u_2')dt\right)}_{\eqqcolon A}(u_2-u_1)\Bigg\|_2 \!\!+ \Bigg\|\underbrace{\left(\eta \int_{0}^1 \nabla_2\nabla_1 F(u_1, u_1' + t(u_2'-u_1'))dt\right)}_{\eqqcolon B}(u_2'-u_1')
        \Bigg\|_2.\\
    \end{align*}
The last inequality is from the triangle inequality and the last equality is from the fundamental theorem of line integrals.
By utilizing the notion of the spectral norm, we have
$$\|A(u_2-u_1)\|_2 \le \sigma_\text{max}(A)\|u_2-u_1\|_2 = \max(|\lambda_\text{min}(A)|,|\lambda_\text{max}(A)|)\|u_2-u_1\|_2,$$
$$\|B(u_2'-u_1')\|_2 \le \sigma_\text{max}(B)\|u_2-u_1\|_2.$$
Due to the second condition in ($*$), $A$ has eigenvalues bounded by $1+\eta L$ and $1+\eta M$, \ie, $$1+\eta L \le \lambda_{\text{min}}(A) \le \lambda_{\text{max}}(A) \le 1+\eta M.$$
Hence, for $0< \eta < -\frac{1}{2L}$, $A$ is positive definite since $\lambda_{\text{min}}(A) \ge 1+\eta L >\frac{1}{2}>0$. Thus, we have 
\begin{equation}
\sigma_{\text{max}}(A) = \lambda_\text{max}(A) \le 1 + \eta M. \nonumber
\end{equation}
Moreover, utilizing the third condition in ($*$), $B$ has singular values bounded above by $\eta M'$, \ie,
\begin{equation}
\sigma_{\text{max}}(B) \le \eta M'. \nonumber
\end{equation}
As a result, we get 
\begin{align*}
\|\mathrm{PD}_F(u_2;u_2',\eta) - &\mathrm{PD}_F(u_1;u_1',\eta)\|_2 
\\&\le 
\|A(u_2-u_1)\|_2 + \|B(u_2'-u_1')\|_2\\
&\le \sigma_\text{max}(A) \|u_2-u_1\|_2 + \sigma_\text{max}(B) \|u_2'-u_1'\|_2\\
&\le (\sigma_\text{max}(A)+ \sigma_\text{max}(B)) \|u_2'-u_1'\|_2 &&(\because \|u_2-u_1\|_2\le \|u_2'-u_1'\|_2)\\
&\le (1 + \eta M + \eta M') \|u_2'-u_1'\|^2.
\end{align*}
Since $0<M'<-M \le -L$, we have $0<1+\eta M + \eta M'<1$ for $0 <\eta < -\frac{1}{2L}$.
Hence for $0 < \eta < -\frac{1}{2L}$ and $\beta = (1+\eta M +\eta M') < 1$, we proved that the partial derivative step satisfies the desired property, concluding the proof.
\end{proof}
From \Cref{prop:contract}, we can derive two important corollaries.
\begin{corollary}
\label{cor:contract1}
    If $F:\mathcal{U}\times\mathcal{U}\to\reals$ satisfies three conditions ($*$), there exists an $\eta >0$ and $0<\beta < 1$ such that the first update rule of \Cref{thm:upd},  $u\leftarrow \mathrm{PD}_F(u;u'=u,\eta)$, 
is a $\beta$-contraction mapping, \ie,
$$\|\mathrm{PD}_F(u_2;u_2,\eta) - \mathrm{PD}_F(u_1;u_1,\eta)\|_2 \le \beta\|u_2-u_1\|_2 $$
for all $u_1, u_2\in\mathcal{U}$.
\end{corollary}
\begin{proof}
    This corollary corresponds to \Cref{prop:contract} of the case that $u_1'=u_1$ and $u_2'=u_2$.
\end{proof}
\begin{corollary}
\label{cor:contract2}
    If $F:\mathcal{U}\times\mathcal{U}\to\reals$ satisfies three conditions ($*$), there exists an $\eta >0$ and $0<\beta < 1$ such that the second update rule of \Cref{thm:upd},  $u\leftarrow (\mathrm{PD}_F(\cdot;u'=u,\eta))^{N_t}(u)$, 
is a $\beta$-contraction mapping, \ie, 
$$\|(\mathrm{PD}_F(\cdot;u'=u_2,\eta))^{N_t}(u_2) - (\mathrm{PD}_F(\cdot;u'=u_1,\eta))^{N_t}(u_1)\|_2 \le \beta\|u_2-u_1\|_2. $$
\begin{proof}
For any given $u_1, u_2\in\mathcal{U}$, we define the following recurrence relation:
\begin{enumerate}
    \item Initialize $u^{(0)}_1 = u_1$, $u^{(0)}_2 = u_2$.
    \item For $i=0, ..., N_t-1$, 
    $$u^{(i+1)}_1 = \mathrm{PD}_F(u^{(i)}_1;u_1,\eta),~~  u^{(i+1)}_2 = \mathrm{PD}_F(u^{(i)}_2;u_2,\eta).$$
\end{enumerate}
Let $0<\beta<1$ and $\eta>0$ be the constants that satisfies the property in \Cref{prop:contract}.
Then, we prove that $\|u_2^{(i)} - u_1^{(i)}\|_2 \le \beta \|u_2-u_1\|_2$ for all $i=1,\ldots,N_t$ by the mathematical induction on $i$.
We considered the initial case of $i=1$ in \Cref{cor:contract1}.
Next, assume that $\|u_2^{(i)} - u_1^{(i)}\|_2 \le \beta \|u_2-u_1\|_2$ for some $i\ge 1$. 
By plugging in $u_1 \leftarrow u_1^{(i)},u_2 \leftarrow u_2^{(i)},u_1'\leftarrow u_1, u_2'\leftarrow u_2$ to \Cref{prop:contract}, we have 
$$\|u^{(i+1)}_1 - u^{(i+1)}_2\|_2 = \|\mathrm{PD}_F(u^{(i)}_1;u_1,\eta) -  \mathrm{PD}_F(u^{(i)}_2;u_2,\eta)\|_2\le \beta \|u_2-u_1\|_2$$
since $\|u_2^{(i)} - u_1^{(i)}\|_2 \le \beta \|u_2-u_1\|_2 < \|u_2-u_1\|_2$.
Thus, we complete the proof by mathematical induction on $i$.
\end{proof}
\end{corollary}
By utilizing these corollaries, we prove the convergence of our proposed update rules when $F$ satisfies three conditions ($*$).

\emph{Proof of \Cref{thm:upd}.} 
For simplicity we notate two update rules by $$\mathrm{upd}_1(u) = \mathrm{PD}_F(u;u'=u,\eta)~~\text{and}~~\mathrm{upd}_2(u) = (\mathrm{PD}_F(\cdot;u'=u,\eta))^{N_t}(u).$$
From the condition in \Cref{thm:upd}, there exists the $u^*\in\mathcal{U}$ such that $u^*=\argmax_{u\in\mathcal{U}} F(u,u^*) = \argmax_{u\in\mathcal{U}} F_{u^*}(u)$.
Hence, $u^*$ satisfies $\left.\frac{\partial}{\partial u }F(u,u^*)\right|_{u=u^*} = 0$. 
Thus, $u^*$ is a fixed point of $\mathrm{upd}_1$ and $\mathrm{upd}_2$. By \Cref{cor:contract1}, there exists $0<\beta<1$ and $\eta>0$ such that $\mathrm{upd}_1$ is a $\beta$-contraction mapping. For any $u\in\mathcal{U}$, we have
$$\|\mathrm{upd}_1(u) - u^*\|_2 = \|\mathrm{upd}_1(u) - \mathrm{upd}_1(u^*) \|_2 \le \beta \|u-u^*\|_2.$$
Thus, by repeatedly applying the update rule $i$ times, we get
$$\|(\mathrm{upd}_1)^i(u) - u^*\|_2 \le \beta^i \|u-u^*\|_2,$$
which converges to $0$ as $i$ goes infinity since $0<\beta <1$.
Also, for the second update rule, we also have $0<\beta<1$ and $\eta>0$ such that $\mathrm{upd}_2$ is a $\beta$-contraction mapping by \Cref{cor:contract2}.
By the same logic, we have 
$$\|(\mathrm{upd}_2)^i(u) - u^*\|_2 \le \beta^i \|u-u^*\|_2,$$
which converges to $0$ as $i$ goes infinity.
Hence, for both update rules, we prove the convergence to $u^*$, finishing the proof.
\qed
\newpage
\subsection{Proof of \Cref{lem:greedypolicy}, \Cref{prop:setgrad}, and \Cref{lem:easy}}
For the proof of these statements, we assume that the set-conditioned policy model has enough representative power to model the policies so that any policy that generate candidates given any set $B$ is corresponding to $\pi_{\theta}^\text{set}$ for some $\theta\in\Theta$. 

In \Cref{alg:exactgreedy} and \Cref{alg:samplesetpolicy}, we inherently assume the unique maximizer at each step for ease of understanding.
However, rigorously, there can be multiple maximizers at each step.
In such cases, we assume that \Cref{alg:exactgreedy} and \Cref{alg:samplesetpolicy} select the candidate uniformly at random among the maximizers at each step.
Consequently, there can be more than one possible solutions from \Cref{alg:exactgreedy}. 
To address this, we define \emph{exact greedy solutions} as follows:
\begin{definition}
    A $n$-subset $B_n\subseteq \mathcal{X}$ is an exact greedy solution if there exists $\xvec_0,\ldots,\xvec_{n-1}\in\mathcal{X}$ that satisfies $$B_n=\{\xvec_0,\ldots,\xvec_{n-1}\}, ~~\text{and} ~~ \forall~ 0\le t \le n-1,~~~
    \xvec_{t}\in\argmax_{\xvec'\in\mathcal{X}}\Delta_a(\xvec'\mid\{\xvec_0,\ldots,\xvec_{t-1}\}),$$
where $\{\xvec_0,\ldots, \xvec_{t-1}\}$ denotes the empty set $\emptyset $ if $t=0$.
\label{def:exactsol}
\end{definition}
In other words, an exact greedy solution is an $n$-subset which can be sampled from \Cref{alg:exactgreedy} with a positive probability. Now, we recall and prove \Cref{lem:greedypolicy} by contradiction as following.
\label{sec:lemsproof}
\begin{lemma}
\label{lem:greedypolicyapp}
$\mathrm{GS}(a,\pi_{\theta^*}^\text{set},n,l)$ samples exact greedy solutions almost surely if $\pi_{\theta^*}$ is the greedy policy.
\end{lemma}
\begin{proof}
(Proof by contradiction.)
Let $\pi_{\theta^*}^\text{set}$ be a greedy policy. According to \Cref{def:greedypolicy}, we have 
\begin{equation}
   \forall~ \theta\in\Theta,~~~~ \mathcal{J}(\theta^*,\theta^*) \ge \mathcal{J}(\theta,\theta^*).
    \label{eq:contrad}
\end{equation} 
Now, assume the opposite of the desired conclusion, that with a positive probability, $\mathrm{GS}(a,\pi_{\theta^*}^\text{set},n,l)$ generates $\xvec_0,\ldots,\xvec_{n-1}$ sequentially by \Cref{alg:samplesetpolicy} and results in a set $B_n=\{\xvec_0,\ldots,\xvec_{n-1}\}$ that is not an exact greedy solution. 
Then,  $\mathrm{GS}(a,\pi_{\theta^*}^\text{set},n,1)$ also generates $\xvec_0,\ldots,\xvec_{n-1}$ with a positive probability. We denote this probability by $c_1>0$, \ie,
\begin{equation}
    c_1\coloneqq \prod_{i=0}^{n-1} \pi_{\theta^*}^\text{set}(\xvec_{i}\mid \{\xvec_0,\ldots,\xvec_{i-1}\}) > 0.
    \label{eq:defc1}
\end{equation} Since $B_n$ is not an exact greedy solution, there exists at least one $0\le t\le n-1$ such that 
\begin{equation}
    \xvec_t \notin \argmax_{\xvec'\in\mathcal{X}} \Delta_a(\xvec'\mid \{\xvec_0,\ldots,\xvec_{t-1}\}),
    \label{eq:defxt}
\end{equation}
according to \Cref{def:exactsol}. For that $t$, let $\pi_\phi^\text{set}$ be the policy that replicates $\pi_{\theta^*}^\text{set}$ for all set $B\subset\mathcal{X}$ except for $\{\xvec_0,\ldots, \xvec_{t-1}\}$, that satisfies
\begin{align*}
    \forall \xvec\in\mathcal{X},~~~~\pi_\phi^\text{set}( \xvec \mid B) = \begin{cases}
        \pi_{\theta^*}^\text{set}( \xvec \mid B) & \text{if $B\neq \{\xvec_0,\ldots,\xvec_{t-1}\}$},\\
        \mathbf{1}_{\{\mathbf{x}_t^*\}}(\xvec) & \text{if $B=\{\xvec_0,\ldots,\xvec_{t-1}\}$},
    \end{cases}
\end{align*}
where $\mathbf{x}_t^*$ be any maximizer of the marginal gain at $t$-th step, \ie, $\mathbf{x}_t^*\in \argmax_{\xvec'\in\mathcal{X}} \Delta_a(\xvec'\mid \{\xvec_0,\ldots,\xvec_{t-1}\})$.
Next, we define the difference in expected return between $\phi$ and $\theta^*$ given a set $B$ as following:
$$C(B)\coloneqq \mathbb{E}_{\xvec\sim\pi_{\phi}^\text{set}(\cdot\mid B)}[\Delta_a(\xvec\mid B)] -     \mathbb{E}_{\xvec\sim\pi_{\theta^*}^\text{set}(\cdot\mid B)}[\Delta_a(\xvec\mid B)].$$
If $B\neq \{\xvec_0,\ldots,\xvec_{t-1}\}$, $\pi_\phi^\text{set}( \xvec \mid B)=\pi_{\theta^*}^\text{set}( \xvec \mid B)$ for all $\xvec\in\mathcal{X}$. Hence, we have a zero difference $C(B)=0$.\\
If $B= \{\xvec_0,\ldots,\xvec_{t-1}\}$, $\pi_\phi^\text{set}( \xvec \mid B)=\mathbf{1}_{\{\mathbf{x}_t^*\}}(\xvec)=\begin{cases}1 & \text{if}~\xvec=\xvec_t^*\\0&\text{otherwise} \end{cases}$ \, for all $\xvec\in\mathcal{X}$. Hence, we have
\begin{align}
 c_2&\coloneqq C(B)\nonumber\\
    &= \Delta_a(\xvec^*_t\mid B) -     \mathbb{E}_{\xvec\sim\pi_{\theta^*}^\text{set}(\cdot\mid B)}[\Delta_a(\xvec\mid B)]\nonumber\\
    &= \left(\max_{\xvec'\in\mathcal{X}}\Delta_a(\xvec'\mid B)\right) -     \mathbb{E}_{\xvec\sim\pi_{\theta^*}^\text{set}(\cdot\mid B)}[\Delta_a(\xvec\mid B)] &&\left(\because \mathbf{x}_t^*\in \argmax_{\xvec'\in\mathcal{X}} \Delta_a(\xvec'\mid B)\right) \nonumber\\
    &=     \mathbb{E}_{\xvec\sim\pi_{\theta^*}^\text{set}(\cdot\mid B)}\left[\left(\max_{\xvec'\in\mathcal{X}}\Delta_a(\xvec'\mid B)\right) - \Delta_a(\xvec\mid B)\right]
    &&\left(\because\left(\max_{\xvec'\in\mathcal{X}}\Delta_a(\xvec'\mid B)\right) \text{is a constant given $B$}\right)\nonumber\\
    &\ge \pi_{\theta^*}^\text{set}(\xvec_t\mid B) \left[\left(\max_{\xvec'\in\mathcal{X}}\Delta_a(\xvec'\mid B)\right) - \Delta_a(\xvec_t\mid B)\right]. &&\left(\because \forall \xvec\in\mathcal{X}, \left(\max_{\xvec'\in\mathcal{X}}\Delta_a(\xvec'\mid B)\right) - \Delta_a(\xvec\mid B)\ge 0 \right)
    \label{eq:longeq}
\end{align}
Here, 
\begin{align}
    \pi_{\theta^*}^\text{set}(\xvec_t\mid B) = \pi_{\theta^*}^\text{set}(\xvec_t\mid\{\xvec_0,\ldots,\xvec_{t-1}\}) \ge \prod_{i=0}^{n-1} \pi_{\theta^*}^\text{set}(\xvec_{i}\mid \{\xvec_0,\ldots,\xvec_{i-1}\}) = c_1 > 0. &&(\because \text{\Cref{eq:defc1}})
    \label{eq:cond1}
\end{align}
Moreover, 
\begin{align}
    \left(\max_{\xvec'\in\mathcal{X}}\Delta_a(\xvec'\mid B)\right) - \Delta_a(\xvec_t\mid B) > 0.&&(\because \text{\Cref{eq:defxt}})
    \label{eq:cond2}
\end{align}
By combining Equations (\ref{eq:longeq}), (\ref{eq:cond1}), and (\ref{eq:cond2}), we finally have $c_2 > 0$.
As a result, we have 
\begin{align*}
    C(B) = \begin{cases}
        0 & \text{if  $B\neq \{\xvec_0,\ldots,\xvec_{t-1}\}$},\\
        c_2 > 0 & \text{if  $B= \{\xvec_0,\ldots,\xvec_{t-1}\}$}.\\
    \end{cases}
\end{align*}
Using this,
\begin{align*}
    \mathcal{J}(\phi,\theta^*) &- \mathcal{J}(\theta^*,\theta^*) \\
    &= \frac{1}{n}\sum_{k=0}^{n-1}\mathbb{E}_{B\sim \mathrm{GS}(a,\pi_{\theta^*}^\text{set},k,1)}[\mathbb{E}_{\xvec\sim\pi_{\phi}^\text{set}(\cdot\mid B)}[\Delta_a(\xvec\mid B)]] -         
    \frac{1}{n}\sum_{k=0}^{n-1}\mathbb{E}_{B\sim \mathrm{GS}(a,\pi_{\theta^*}^\text{set},k,1)}[\mathbb{E}_{\xvec\sim\pi_{\theta^*}^\text{set}(\cdot\mid B)}[\Delta_a(\xvec\mid B)]]\\
    &= \frac{1}{n}\sum_{k=0}^{n-1}\mathbb{E}_{B\sim \mathrm{GS}(a,\pi_{\theta^*}^\text{set},k,1)}[C(B)]\\
    &=\frac{1}{n} \underbrace{(\mathrm{GS}(a,\pi_{\theta^*}^\text{set},t,1))(\{\xvec_0,\ldots,\xvec_{t-1}\})}_{\text{Probability of sampling $\{\xvec_0,\ldots\xvec_{t-1}\}$}} c_2\\
    &\ge \frac{1}{n} \underbrace{\prod_{i=0}^{t-1} \pi_{\theta^*}^\text{set}(\xvec_{i}\mid \{\xvec_0,\ldots,\xvec_{i-1}\})}_{\text{Probability of sampling $\xvec_0\rightarrow\cdots\rightarrow\xvec_{t-1}$}} c_2
    \ge \frac{1}{n} \underbrace{\prod_{i=0}^{n-1} \pi_{\theta^*}^\text{set}(\xvec_{i}\mid \{\xvec_0,\ldots,\xvec_{i-1}\})}_{\text{Probability of sampling $\xvec_0\rightarrow\cdots\rightarrow\xvec_{n-1}$}} c_2 = \frac{c_1c_2}{n}> 0.
\end{align*}
 Thus, we finally have $\mathcal{J}(\phi,\theta^*) > \mathcal{J}(\theta^*,\theta^*)$ which contradicts to \Cref{eq:contrad}. In conclusion, we complete the proof by contradiction.
\end{proof}
Next, we recall and prove \Cref{prop:setgrad}.
\begin{proposition}
\label{lem:setgradapp} (Policy gradient for $\mathcal{J}$)
For any baseline set function $b:2^\mathcal{X}\to\reals$ and the number of episodes $N_e$,

$$\hat{g}=\frac{1}{N_e}\sum_{j=0}^{N_e-1}(\Delta_a(\xvec^{(j)}\mid B) - b(B)) \nabla_\theta \pi^\text{set}_\theta(\xvec^{(j)}\mid B), $$

is an unbiased MC estimator of  the partial derivative $\frac{\partial}{\partial \theta} \mathcal{J}(\theta,\theta')$ where $B\sim \frac{1}{n}\sum_{k=0}^{n-1}\mathrm{GS}(a,\pi_{\theta'}^\text{set},k,1)$.
\end{proposition}
\begin{proof}
This proposition is a variation of policy gradients and our proof also takes the same idea of the log-derivative trick for obtaining MC estimator of the derivative \citep{reinforce}. 
First of all, we can combine $n$ expectations in $\mathcal{J}(\theta,\theta')$ by utilizing the mixture of distributions $\frac{1}{n}\sum_{k=0}^{n-1}\mathrm{GS}(a,\pi_{\theta'}^\text{set},k,1)$ as follows:
    \begin{align*} 
        \mathcal{J}(\theta,\theta')&=\frac{1}{n}\sum_{k=0}^{n-1}\mathbb{E}_{B\sim \mathrm{GS}(a,\pi_{\theta'}^\text{set},k,1)}[\mathbb{E}_{\xvec\sim\pi_\theta^\text{set}(\cdot\mid B)}[\Delta_a(\xvec\mid B)]]\\
        &=\mathbb{E}_{B\sim \frac{1}{n}\sum_{k=0}^{n-1}\mathrm{GS}(a,\pi_{\theta'}^\text{set},k,1)}[\mathbb{E}_{\xvec\sim\pi_\theta^\text{set}(\cdot\mid B)}[\Delta_a(\xvec\mid B)]].
    \end{align*}
Since the mixture of the distributions $\frac{1}{n}\sum_{k=0}^{n-1}\mathrm{GS}(a,\pi_{\theta'}^\text{set},k,1)$ is independent to $\theta$, we denote this distribution by $p$ for simplicity.
Then, our goal is to obtain the estimator of partial derivative:
\begin{align}
        \frac{\partial}{\partial \theta} \mathcal{J}(\theta,\theta')
        &=\frac{\partial}{\partial \theta} \mathbb{E}_{B\sim p}[\mathbb{E}_{\xvec\sim\pi_\theta^\text{set}(\cdot\mid B)}[\Delta_a(\xvec\mid B)]]\nonumber\\
    &=
\frac{\partial}{\partial \theta}  \left(\sum_{B\subset\mathcal{X}, |B|\le n} p(B) \sum_{\xvec\in\mathcal{X}}\Delta_a(\xvec\mid B) \pi_\theta^\text{set}(\xvec\mid B) \right)\nonumber\\
&= \sum_{B\subset\mathcal{X}, |B|\le n} p(B)  \left(\sum_{\xvec\in\mathcal{X}} \Delta_a(\xvec\mid B)\nabla_\theta \pi_\theta^\text{set}(\xvec\mid B) \right)\nonumber\\
&= \sum_{B\subset\mathcal{X}, |B|\le n} p(B)  \left(\sum_{\xvec\in\mathcal{X}}  \Delta_a(\xvec\mid B) \pi_\theta^\text{set}(\xvec\mid B)\nabla_\theta \log \pi_\theta^\text{set}(\xvec\mid B)\right) &&\left(\because \nabla_\theta \log \pi_\theta^\text{set} = \frac{\nabla_\theta \pi_\theta^\text{set}}{\pi_\theta^\text{set}}\right)\nonumber\\
&= \sum_{B\subset\mathcal{X}, |B|\le n} p(B)  \mathbb{E}_{\xvec\sim\pi_\theta^\text{set}(\cdot\mid B)}\left[\Delta_a(\xvec\mid B) \nabla_\theta \log \pi_\theta^\text{set}(\xvec\mid B)\right] \nonumber\\
&=\mathbb{E}_{B\sim p} \left[\mathbb{E}_{\xvec\sim\pi_\theta^\text{set}(\cdot\mid B)}\left[\Delta_a(\xvec\mid B) \nabla_\theta \log \pi_\theta^\text{set}(\xvec\mid B)\right] \right].
\label{eq:policygrad}
\end{align}
For the next step, we prove that following equality holds for any baseline set function $b:2^\mathcal{X}\to\reals$:
\begin{align}
\mathbb{E}_{B\sim p} \left[\mathbb{E}_{\xvec\sim\pi_\theta^\text{set}(\cdot\mid B)}\left[b(B) \nabla_\theta \log \pi_\theta^\text{set}(\xvec\mid B)\right] \right]=0.
\label{eq:base}
\end{align}
For any $B\subset \mathcal{X}$, we have 
\begin{align*}
    \mathbb{E}_{\xvec\sim\pi_\theta^\text{set}(\cdot\mid B)} [b(B)\nabla_\theta \log \pi_\theta^\text{set}(\xvec\mid B)] 
    &= b(B)      \mathbb{E}_{\xvec\sim\pi_\theta^\text{set}(\cdot\mid B)} [\nabla_\theta \log \pi_\theta^\text{set}(\xvec\mid B)] \\
    &= b(B) \sum_{\xvec\in \mathcal{X}}\pi_\theta^\text{set}(\xvec\mid B)\nabla_\theta \log \pi_\theta^\text{set}(\xvec\mid B)\\
    &= b(B) \sum_{\xvec\in \mathcal{X}}\nabla_\theta \pi_\theta^\text{set}(\xvec\mid B)\\
    &= b(B) \nabla_\theta\left(\sum_{\xvec\in \mathcal{X}}\pi_\theta^\text{set}(\xvec\mid B)\right) =  b(B) \nabla_\theta (1) = 0.
\end{align*}
Thus, we get \Cref{eq:base}. By combining \Cref{eq:policygrad} and \Cref{eq:base}, we finally achieve 
\begin{align*}
        \frac{\partial}{\partial \theta} \mathcal{J}(\theta,\theta')
&=\mathbb{E}_{B\sim p} \left[\mathbb{E}_{\xvec\sim\pi_\theta^\text{set}(\cdot\mid B)}\left[(\Delta_a(\xvec\mid B) -b(B))\nabla_\theta \log \pi_\theta^\text{set}(\xvec\mid B)\right] \right].
\end{align*}
Finally, we derive an unbiased MC estimator $\hat{g}$ of the parital derivative $ \frac{\partial}{\partial \theta} \mathcal{J}(\theta,\theta')$ by 
$$\hat{g}=\frac{1}{N_e}\sum_{j=0}^{N_e-1}(\Delta_a(\xvec^{(j)}\mid B) - b(B)) \nabla_\theta \log \pi^\text{set}_\theta(\xvec^{(j)}\mid B), $$
where $B\sim p$ and $\xvec^{(j)}\sim \pi_\theta^\text{set}(\cdot\mid B)$.
\end{proof}
Finally, we recall and prove \Cref{lem:easy}.
\begin{lemma}
\label{lem:easyapp}
    For any $B,B'\subset\mathcal{X}$ satisfying $\tilde{f}(B)=\tilde{f}(B')$, $\Delta_a(\xvec\mid B) = \Delta_a(\xvec\mid B')$ for all $\xvec\in\mathcal{X}$ if $a$ is given by HVI and $\tilde{f}$ is deterministic surrogate function.
\end{lemma}
\begin{proof}
We prove the consequence directly from the definition of HVI as follows:
\begin{align*}
    \Delta_a(\xvec\mid B) = \mathbf{HVI}(B;\tilde{\fvec},\tilde{\mathcal{P}},\mathbf{r}_\text{ref}) &= \mathbf{HV}(\tilde{f}(B)\cup\tilde{f}(\tilde{\mathcal{P}});\mathbf{r}_\text{ref}) - \mathbf{HV}(\tilde{f}(\tilde{\mathcal{P}});\mathbf{r}_\text{ref})\\
    &= \mathbf{HV}(\tilde{f}(B')\cup\tilde{f}(\tilde{\mathcal{P}});\mathbf{r}_\text{ref}) - \mathbf{HV}(\tilde{f}(\tilde{\mathcal{P}});\mathbf{r}_\text{ref})\\
    &= \mathbf{HVI}(B';\tilde{\fvec}, \tilde{\mathcal{P}}, \mathbf{r}_\text{ref})= \Delta_a(\xvec\mid B').
\end{align*}
\end{proof}

\subsection{Detailed Explanation of Bounds for Approximated Greedy Algorithm}
\label{app:boundsdetail}
In this section, we provide a more formal explanation on bounds for approximated greedy algorithm proposed in \Cref{sec:bounds}.
First, we introduce the \emph{monotone submodularity}.
\begin{definition} (Monotone Submodularity)
    A real-valued set function $a:2^\mathcal{X}\to\reals$ is submodular if the inequality $\Delta_a(\xvec \mid B') \ge \Delta_a(\xvec \mid B)$ holds for all $B' \subset B \subset \mathcal{X}$ and $\mathbf{x}\in\mathcal{X}\setminus B$. 
    Additionally, a set function $a$ is monotone if $\Delta_a(\xvec \mid B) \ge 0$ for all $B \subset \mathcal{X}$ and $\xvec \in \mathcal{X}\setminus B$.  
    Finally, a set function $s$ is non-negative if $a(B)\ge 0$ for all $B\subset \mathcal{X}$. 
\end{definition}
In essence, a monotone submodularity is characterized by a consistently non-negative marginal gain that diminishes as the augmenting set enlarges. 
For a non-negative monotone submodular function $a$, the exact greedy algorithm is known to guarantee a $(1 - 1/e)$-approximation to the optimal $n$-subset $B_n^*$, \ie, $a(B_n)\ge (1-1/e) a(B_n^*)$ \citep{exactgreedybound}. 

In \Cref{thm:modbound}, we extend this bound to the approximated greedy algorithm when $a$ is any monotone near-submodular set function using the notion of \emph{submodularity ratio}, which is formally defined as follows.
\begin{definition} (Submodularity Ratio)
\label{def:submodind}
Let $a:2^\mathcal{X}\to\reals$ be a monotone set function. For $B\subset \mathcal{X}$ and $n\ge 1$, the submodularity ratio $\gamma_{B,n}(a)$ is defined as 
$$\gamma_{B,n}(a)\coloneqq \min_{S\subset \mathcal{X}, B'\subset B\setminus S, |S|\le n}\frac{\sum_{\xvec\in S}\Delta_a(\xvec\mid B')}{a(B'\cup S)-a(B')}\le 1,$$
where we define $0/0\coloneqq1$ ($S=\emptyset$ case).
\end{definition}
The submodularity ratio measures the extent to which the function $a$ exhibits submodularity \citep{submodindexdef}. 
\subsubsection{Proof of \Cref{thm:modbound}}
\label{sec:modproof}
To start, we recall \Cref{thm:modbound}. 
\begin{theorem}
\label{thm:appmodbound}
    Let $a:2^\mathcal{X}\to\reals$ be a non-negative monotone set function. If $\mathcal{A}$ is an $\alpha$-approximation algorithm, the resulting solution $B_n$ of \Cref{alg:approxgreedy} is an $(1-1/e^{\alpha\gamma_{B_n,n}(a)})$-approxmiation to the optimal $n$-subset $B_n^*$, \ie, $a(B_n)\ge (1-1/e^{\alpha\gamma_{B_n,n}(a)}) a(B_n^*)$.
\end{theorem}
Next, we summarize notations for the proof as follows:
\begin{table}[hbt!]
    \centering
    \caption{Notations for the proof of \Cref{thm:modbound}.}
    \begin{tabular}{ll}
    \toprule
    Notation& Definition\\
    \midrule
    $\mathcal{A}$ & an $\alpha$-approximation algorithm.\\
    $n\in\mathbb{N}$ & a cardinality constraint.\\
    $a:2^\mathcal{X}\to\reals$ & a non-negative monotone objective set function.\\
      $\xvec_i^{\mathcal{A}}\in\mathcal{X}$   & the candidate appended at $i$-th step of \Cref{alg:approxgreedy} with an algorithm $\mathcal{A}$.  \\
    $B_i^\mathcal{A}=\{\xvec_j^{\mathcal{A}}\mid 0\le j \le i-1\}\subset\mathcal{X}$   & the $i$-subset constructed by \Cref{alg:approxgreedy} with an algorithm $\mathcal{A}$.\\
      $B_i^*\in\argmax_{B\subset \mathcal{X}, |B|=i} a(B)$   & the optimal $i$-subset of $a$.  \\ 
    \bottomrule
    \end{tabular}
    \label{tab:notationmodproof}
\end{table}

Our proof is an extension of the proof by \citet{submodindexbound}. 
Using the notion of the submodularity ratio, we first prove the following lemma.
\begin{lemma}
\label{lem:mod1}
For any  $0\le i< n$, the following inequality holds:
    $$\Delta_a(\xvec_i^\mathcal{A}\mid B_i^\mathcal{A}) \ge \frac{\alpha\gamma_{B_n^\mathcal{A},n}(a)}{n}(a(B_n^*) - a(B_i^\mathcal{A})).$$
\end{lemma}
\begin{proof}
Let $S_i\coloneqq B_n^*\setminus B_i^\mathcal{A}$. Then,
    \begin{align*}
        \Delta_a(\xvec_i^\mathcal{A}\mid B_i^\mathcal{A}) 
        &\ge \alpha \max_{\xvec\in\mathcal{X}} \Delta_a(\xvec \mid B_i^\mathcal{A})  && (\because \mathcal{A} ~\text{is an}~ \alpha\text{-approximation algorithm}.)\\
        &\ge \alpha \frac{\sum_{\xvec\in S_i} \Delta_a(\xvec\mid B_i^\mathcal{A})}{|S_i|} && (\because \text{maximality})\\
        &\ge \alpha \frac{\gamma_{B_n^\mathcal{A}, n}(a)}{|S_i|} (a(B_i^\mathcal{A}\cup S_i) - a(B_i^\mathcal{A})) &&(\because \text{\Cref{def:submodind}}, B_i^\mathcal{A}\subset B_n^\mathcal{A}\setminus S_i, |S_i|\le |B_n^*|= n)\\
        &\ge  \frac{\alpha\gamma_{B_n^\mathcal{A},n}(a)}{n}(a( B_i^\mathcal{A}\cup S_i) - a(B_i^\mathcal{A})) &&(\because |S_i|\le |B_n^*| = n)\\
        &\ge  \frac{\alpha\gamma_{B_n^\mathcal{A},n}(a)}{n}(a(B_n^*)-a(B_i^\mathcal{A})). &&(\because  B_n^*\subseteq S_i\cup B_i^\mathcal{A}, \text{monotonicity of }a)\\
    \end{align*}
\end{proof}
For the next step, we prove the following lemma using \Cref{lem:mod1}.
\begin{lemma}
\label{lem:mod2}
For any $i\ge 0$, the following inequality holds:
    $$a(B_n^*) - a(B_{i+1}^\mathcal{A}) \le \left(1-\frac{\alpha\gamma_{B_n^\mathcal{A},n}(a)}{n}\right)(a(B_n^*) - a(B_{i}^\mathcal{A})).$$
\end{lemma}
\begin{proof}
    \begin{align*}
        a(B_n^*) - a(B_{i+1}^\mathcal{A}) &= a(B_n^*) - (a(B_{i}^\mathcal{A}) + \Delta_a(\xvec_i^\mathcal{A}\mid B_i^\mathcal{A}))\\
        &\le (a(B_n^*) - a(B_{i}^\mathcal{A})) - \frac{\alpha\gamma_{B_n^\mathcal{A},n}(a)}{n}(a(B_n^*) - a(B_i^\mathcal{A})) && (\because \text{\Cref{lem:mod1}})\\
        &= \left(1-\frac{\alpha\gamma_{B_n^\mathcal{A},n}(a)}{n}\right)(a(B_n^*) - a(B_i^\mathcal{A})).
    \end{align*}
\end{proof}
Finally, we prove \Cref{thm:modbound}.\\

\emph{Proof of \Cref{thm:modbound}.} By combining \Cref{lem:mod2} with $i=0,\ldots,n-1$, we get
\begin{align*}
    a(B_n^*)-a(B_n^\mathcal{A}) &\le \left(1-\frac{\alpha\gamma_{B_n^\mathcal{A},n}(a)}{n}\right)^n (a(B_n^*)-a(B_0^\mathcal{A}))\\
    &\le\left(1-\frac{\alpha\gamma_{B_n^\mathcal{A},n}(a)}{n}\right)^n a(B_n^*). && (\because \text{non-negativity of }a)
\end{align*}
Finally, we get the desired bound as follows:
\begin{align*}
    a(B_n^\mathcal{A}) &\ge a(B_n^*) - \left(1-\frac{\alpha\gamma_{B_n^\mathcal{A},n}(a)}{n}\right)^n a(B_n^*)\\
    &= \left(1-\left(1-\frac{\alpha\gamma_{B_n^\mathcal{A},n}(a)}{n}\right)^n \right) a(B_n^*)\\
    &\ge \left(1-\frac{1}{e^{\alpha\gamma_{B_n^\mathcal{A},n}(a)}}\right) a(B_n^*). 
    &&(\because (1-1/t)^t\le 1/e \text{~for any~} t\ge1)
\end{align*}
\qed

\subsubsection{Proof of \Cref{thm:divmodbound}}
\label{sec:divmodproof}
First, we recall \Cref{thm:divmodbound}.
\begin{theorem}
\label{thm:appdivmodbound}
    Let $s$ be a non-negative monotone set function and $\mathrm{div}$ be a sum-dispersion function. If $\mathcal{A}$ is an $\alpha$-approximation algorithm, 
    \Cref{alg:approxgreedy} with the set function $a = s/2+ \lambda \mathrm{div}$ returns $B_n$, an $(\alpha\hat{\gamma}/2)$-approximation to the optimal $n$-subset $B_n^*$ of $(s+\lambda \mathrm{div})$, where $\hat{\gamma}\coloneqq \gamma_{B_n^*\cup B_n,n}(s)$.
\end{theorem}
\Cref{thm:divmodbound} suggests a bound for non-oblivious variation of the approximated greedy algorithm which guide the algorithm with the set function $a=s/2 + \lambda \mathrm{div}$ that is different to the actual objective function $s + \lambda \mathrm{div}$ \citep{nonoblivious}.
Our bound extends the theoretical bound of the non-oblivious exact greedy algorithm proved by \citet{borodin} when the objective set function is the sum of a submodular function and a sum-dispersion function.
Our proof combines the ideas from \citet{borodin} and \citet{submodindexbound}. To start, we define some notations as following:

\begin{table}[hbt!]
    \centering
    \caption{Notations for the proof of \Cref{thm:divmodbound}.}
    \resizebox{\columnwidth}{!}{
    \begin{tabular}{ll}
    \toprule
    Notation& Definition\\
    \midrule
    $\mathcal{A}$ & an $\alpha$-approximation algorithm.\\
    $n\in\mathbb{N}$ & a cardinality constraint.\\
    $s:2^\mathcal{X}\to\reals$ & a non-negative monotone set function.\\
    $\mathrm{div}:2^\mathcal{X}\to\reals$ & a dispersion function defined as $\mathrm{div}(B)=\frac{1}{2}\sum_{\xvec\in B}\sum_{\xvec'\in B} d(\xvec, \xvec')$ for a metric $d.$\\
    $\lambda > 0$ & a real-valued coefficient that controlls tradeoff between $s$ and $\mathrm{div}$.\\
    $a = s + \lambda \mathrm{div}$ & the actual objective set function to optimize.\\
    $\tilde{a} = s/2 + \lambda \mathrm{div}$ & the set function to optimize during subproblems of \Cref{alg:approxgreedy}.\\
      $\xvec_i^{\mathcal{A}}\in\mathcal{X}$   & the candidate appended at $i$-th step of \Cref{alg:approxgreedy} with $\tilde{a}=\frac{1}{2}s+\lambda\mathrm{div}$ and  $\mathcal{A}$.  \\
    $B_i^\mathcal{A}=\{\xvec_j^{\mathcal{A}}\mid 0\le j \le i-1\}\subset\mathcal{X}$   & the $i$-subset constructed by \Cref{alg:approxgreedy} with $\tilde{a}=\frac{1}{2}s+\lambda\mathrm{div}$ and  $\mathcal{A}$.\\
      $B_i^*\in\argmax_{B\subset \mathcal{X}, |B|=i} a(B)$   & the optimal $i$-subset of $a=s+\lambda \mathrm{div}$.  \\ 
      $\hat{\gamma}\coloneqq \gamma_{B_n^*\cup B_n^\mathcal{A},n}(s)$ & the submodularity index.\\
    \bottomrule
    \end{tabular}}
    \label{tab:notationdivmodproof}
\end{table}

For notational simplicity, we define $d(B,B')\coloneqq \sum_{\xvec \in B}\sum_{\xvec'\in B'} d(\xvec,\xvec')$ for any $B,B'\subset \mathcal{X}$. Then, $\mathrm{div}(B)=\frac{1}{2}d(B,B)$.
We introduce two lemmas from the previous works \citep{ravi,borodin}. For the completeness, we contain the proof of these lemmas.
\begin{lemma} \citep{ravi}
\label{lem:ravi}
    For a given metric function $d:\mathcal{X}\times\mathcal{X}\to\reals$ and two disjoint sets $B, B' \subset \mathcal{X},$ we have the following inequality:
        $(|B'|-1)d(B,B')\ge |B|\mathrm{div}(B').$
\end{lemma}
\begin{proof}
{\begin{align*}
                |B| \mathrm{div}(B')
                &= \frac{1}{2}|B| d(B',B') \\
                &= \frac{1}{2}\sum_{\xvec\in B}\sum_{\xvec'\in B'} \sum_{\xvec''\in B'} d(\xvec',\xvec'') \\
                &= \frac{1}{2}\sum_{\xvec\in B}\sum_{\xvec'\in B'} \sum_{\xvec''\in B'\setminus\{\xvec'\}} d(\xvec',\xvec'') &&(\because d(\xvec',\xvec')=0)\\
                &\le \frac{1}{2} \sum_{\xvec\in B}\sum_{\xvec'\in B'} \sum_{\xvec''\in B'\setminus\{\xvec'\}} (d(\xvec,\xvec') + d(\xvec,\xvec'')) &&(\because \text{triangle inequality})\\
                &= \frac{1}{2} \sum_{\xvec\in B}\sum_{\xvec'\in B'} \sum_{\xvec''\in B'\setminus\{\xvec'\}} d(\xvec,\xvec') + \frac{1}{2} \sum_{\xvec\in B}\sum_{\xvec''\in B'} \sum_{\xvec'\in B'\setminus\{\xvec''\}} d(\xvec,\xvec'') \\
                &=\frac{|B'|-1}{2} \sum_{\xvec\in B}\sum_{\xvec'\in B'} d(\xvec,\xvec') + 
                \frac{|B'|-1}{2} \sum_{\xvec\in B}\sum_{\xvec''\in B'} d(\xvec,\xvec'')\\
                &= (|B'|-1)  \sum_{\xvec\in B}\sum_{\xvec'\in B'} d(\xvec,\xvec')
                = (|B'|-1) d(B,B').
            \end{align*}}
\normalsize
\end{proof}
\begin{lemma}\citep{borodin}
\label{lem:borodin}
            For $1\le i\le n$, let $U=B_n^*\cap B_i^\mathcal{A}$, $V=B_i^\mathcal{A}-U$, and $W=B_n^*-U$. If $|W|>1$, we have the following inequality:
        $$d(B_i^\mathcal{A},W)\ge \frac{i|W|}{n(n-1)}\mathrm{div}(B_n^*).$$
\end{lemma}
\begin{proof}
                Using \Cref{lem:ravi},
            \begin{align}
                (|W|-1)d(V,W) \ge |V|\mathrm{div}(W),\label{eq:bor1}\\
                (|W|-1)d(U,W) \ge |U|\mathrm{div}(W),\label{eq:bor2}\\
                (|U|-1)d(U,W) \ge |W|\mathrm{div}(U).\label{eq:bor3}
            \end{align}
            Also,
            \begin{align}
                \mathrm{div}(B^*_n) &= \frac{1}{2} \sum_{\xvec\in B^*_n} \sum_{\xvec'\in B^*_n} d(\xvec,\xvec') \nonumber \\
                &= \frac{1}{2} \sum_{\xvec\in U \cup W} \sum_{\xvec'\in U\cup W} d(\xvec,\xvec') \nonumber  && (\because B_n^*=U\cup W)\\
                &= \frac{1}{2} \left(\sum_{\xvec\in U}\sum_{\xvec'\in U} d(\xvec,\xvec')\right) + \left(\sum_{\xvec\in U}\sum_{\xvec'\in W} d(\xvec,\xvec')\right) + \frac{1}{2} \left(\sum_{\xvec\in W}\sum_{\xvec'\in W} d(\xvec,\xvec')\right) \nonumber\\
                &= \mathrm{div}(U) + d(U,W) + \mathrm{div}(W).
                \label{eq:bor4}
            \end{align}
                By combining four equations as $$\text{\Cref{eq:bor1}}\times \frac{1}{|W|-1}+\text{\Cref{eq:bor2}}\times\frac{|W|-|V|}{n(|W|-1)} + \text{\Cref{eq:bor3}}\times \frac{i}{n(n-1)} + \text{\Cref{eq:bor4}}\times \frac{i|W|}{n(n-1)},$$ we have the following inequality:
        \begin{equation}
        \label{eq:combine}
            d(U,W) + d(V,W) - \frac{i|W|(n-|W|)}{n(n-1)(|W|-1)}\mathrm{div}(W) \ge \frac{i|W|}{n(n-1)}\mathrm{div}(B^*_n).
        \end{equation}
        Hence, we have the desired result as follows:
        \begin{align*}
            d(B_i^\mathcal{A}, W) &= d(U\cup V, W) \\
            &= d(U,W) + d(V,W) && (\because U\cap V = \emptyset) \\ 
            &\ge d(U,W) + d(V,W) - \frac{i|W|(n-|W|)}{n(n-1)(|W|-1)}\mathrm{div}(W) && (\because 1<|W|\le |B_n^*| = n) \\ 
            &\ge \frac{i|W|}{n(n-1)}\mathrm{div}(B^*_n). && (\because  \text{\Cref{eq:combine}})
        \end{align*}
\end{proof}
For the convenience, we introduce the following lemma:
\begin{lemma}
\label{lem:ata}
    For any $B\subset \mathcal{X}$ and $\xvec\in\mathcal{X}$, the following inequality holds: $ \frac{1}{2}\Delta_a(\xvec\mid B) \le \Delta_{\tilde{a}}(\xvec\mid B) \le \Delta_a (\xvec\mid B). $
\end{lemma}
\begin{proof}
\begin{align*}
    \Delta_{\tilde{a}}(\xvec\mid B) = \frac{1}{2}\Delta_s(\xvec\mid B) + \lambda \Delta_{\mathrm{div}}(\xvec\mid B) &\le \Delta_s(\xvec\mid B) + \lambda \Delta_{\mathrm{div}}(\xvec\mid B) && (\because \text{monotoniciy of $s$})\\
    &=\Delta_a(\xvec\mid B),
\end{align*}
\begin{align*}
    \Delta_{\tilde{a}}(\xvec\mid B) = \frac{1}{2}\Delta_s(\xvec\mid B) + \lambda \Delta_{\mathrm{div}}(\xvec\mid B) &\ge \frac{1}{2}\left(\Delta_s(\xvec\mid B) + \lambda \Delta_{\mathrm{div}}(\xvec\mid B)\right) && (\because \text{monotoniciy of $\mathrm{div}$})\\
    &=\frac{1}{2}\Delta_a(\xvec\mid B).
\end{align*}
\end{proof}
\emph{Proof of \Cref{thm:divmodbound}.}\\
        \textbf{(Case 1: $n=1$)} 
        \\
        Let $\xvec^\dagger\in\mathcal{X}$ be the maximizer of $\Delta_{\tilde{a}}(\cdot \mid \emptyset)$ and $\xvec^*\in \mathcal{X}$ be the maximizer of $\Delta_{a}(\cdot\mid \emptyset)$, \ie $\{\xvec^*\}=B_1^*$. Then, we have
        \begin{align*}
        a(B_1^\mathcal{A}) 
        = a(\{\xvec_0^\mathcal{A}\})
        &= s(\{\xvec_0^\mathcal{A}\}) + \lambda \mathrm{div}(\{\xvec_0^\mathcal{A}\})\\
        &\ge \frac{1}{2}s(\{\xvec_0^\mathcal{A}\}) + \lambda \mathrm{div}(\{\xvec_0^\mathcal{A}\}) &&(\because \text{non-negativity of}~ s)\\
        &= \tilde{a}(\{\xvec_0^\mathcal{A}\})\\
        &= \Delta_{\tilde{a}} (\xvec_0^\mathcal{A}\mid\emptyset)\\
        &\ge \alpha \Delta_{\tilde{a}}(\xvec^\dagger\mid\emptyset) && (\because \xvec_0^\mathcal{A} \text{~is an $\alpha$-approximation})\\
        &\ge \alpha \Delta_{\tilde{a}}(\xvec^*\mid\emptyset) &&(\because \text{optimality of}~\xvec^\dagger)\\
        &= \alpha \left(\frac{1}{2} s(\{\xvec^*\}) + \mathrm{div}(\{\xvec^*\})\right)\\
        &\ge \frac{\alpha}{2} (s(\{\xvec^*\}) + \mathrm{div}(\{\xvec^*\})) && (\because \text{non-negativity of } \mathrm{div})\\
        & = \frac{\alpha}{2} a(\{\xvec^*\})
         =  \frac{\alpha}{2} a(B_1^*)\ge \frac{\alpha\hat{\gamma}}{2}a(B_1^*). 
        \end{align*}
        
(\textbf{Case 2}: $n>1$)\\ For any $1\le i < n$, let $U=B^*_n\cap B_i^\mathcal{A}$, $V=B_i^\mathcal{A}-U$, and $W=B^*_n-U$ as in \Cref{lem:ravi}.\\
(\textbf{Case 2.a}: $n>1$ and $|W|=1$)\\
In this case, we have $i=n-1$ and $B_i^\mathcal{A} \subset B^*_n$ since $i<n$.
Let $\xvec^*\in B_{n-1}^\mathcal{A}$ be the element that is not in $B_n^*$, \ie, $\{\xvec^*\}=B_n^*\setminus B_{n-1}^\mathcal{A}$.
Let $\xvec^\dagger\in \mathcal{X}\setminus B_{n-1}^\mathcal{A}$ be the maximizer of $\Delta_{\tilde{a}}(\cdot\mid B_{n-1}^\mathcal{A})$. Then,
        \begin{align*}
            a(B_n^\mathcal{A}) &= a(B_{n-1}^\mathcal{A}) +\Delta_a(\xvec_{n-1}^\mathcal{A}\mid B_{n-1}^\mathcal{A})
            && (\because B_n^\mathcal{A} = B_{n-1}^\mathcal{A} \cup \{\xvec_{n-1}^\mathcal{A}\})\\
                 &\ge a(B_{n-1}^\mathcal{A}) + \alpha \Delta_a(\xvec^\dagger \mid B_{n-1}^\mathcal{A})  && (\because \xvec_{n-1}^\mathcal{A} \text{~is an $\alpha$-approximation})\\
                 &\ge a(B_{n-1}^\mathcal{A}) + \alpha \Delta_{\tilde{a}}(\xvec^\dagger \mid B_{n-1}^\mathcal{A}) && (\because \text{\Cref{lem:ata}})\\
                 &\ge a(B_{n-1}^\mathcal{A}) + \alpha \Delta_{\tilde{a}}(\xvec^*\mid B_{n-1}^\mathcal{A}) &&(\because \text{optimality of}~\xvec^\dagger)\\
                 &\ge a(B_{n-1}^\mathcal{A}) + \frac{\alpha }{2} \Delta_{a}(\xvec^*\mid B_{n-1}^\mathcal{A})  && 
                 (\because \text{\Cref{lem:ata}})\\
                 &\ge \frac{\alpha}{2} (a(B_{n-1}^\mathcal{A}) +  \Delta_{a}(\xvec^*\mid B_{n-1}^\mathcal{A})) && (\because \text{non-negativity of $a$})\\
                 &= \frac{\alpha }{2} a(B^*_n) \ge \frac{\alpha \hat{\gamma}}{2} a(B_n^*). && (\because \text{\Cref{def:submodind}})
        \end{align*}
(\textbf{Case 2.b}: $n>1$ and $|W|>1$)\\
Now we can consider the case that $n > 1$ and $|W|>1$. Using \Cref{lem:borodin}, we have
\begin{equation}
    d(B_i^\mathcal{A},W)\ge \frac{i|W|}{n(n-1)}\mathrm{div}(B^*_n). 
    \label{eq:modiv1}
\end{equation}
Using the monotonicity of $s$ and the fact that $B_i^\mathcal{A}\cup W \subset B_i^\mathcal{A}\cup B_n^* \subset B_n^\mathcal{A}\cup B_n^*$, $|W|\le |B_n^*|=n$, and $W\cap B_i^\mathcal{A}=\emptyset$ with \Cref{def:submodind}, we have
\begin{equation}
    \sum_{\xvec\in W} \Delta_{s}(\xvec\mid B_i^\mathcal{A}) \ge \hat{\gamma}\left(s( B_i^\mathcal{A}\cup W) - s(B_i^\mathcal{A}) \right)\ge \hat{\gamma}(s(B_n^*) - s(B_n^\mathcal{A})).
    \label{eq:modiv2}
\end{equation}
        Thus, 
        \begin{align}
            \sum_{\xvec\in W} \Delta_{\tilde{a}}(\xvec\mid B_i^\mathcal{A}) 
            &= \sum_{\xvec\in W} \left(\frac{1}{2}\Delta_s(\xvec\mid B_i^\mathcal{A}) + \lambda d(\xvec, B_i^\mathcal{A})\right)\nonumber\\
            &= \sum_{\xvec\in W} \frac{1}{2}\Delta_s(\xvec\mid B_i^\mathcal{A}) + \lambda d(W, B_i^\mathcal{A})\nonumber\\
            &\ge \frac{\hat{\gamma}}{2}(s(B_n^*) - s(B_n^\mathcal{A})) + \frac{i\lambda|W|}{n(n-1)} \mathrm{div}(B_n^*). &&(\because \text{\Cref{eq:modiv1} and \Cref{eq:modiv2}})
            \label{eq:modivprev}
        \end{align}
By utilizing the previous inequality, we have
                \begin{align*}
                \Delta_{\tilde{a}}(\xvec_{i}^\mathcal{A}\mid B_i^\mathcal{A}) 
                &\ge \alpha \max_{\xvec\in \mathcal{X}\setminus B_i^\mathcal{A}}  \Delta_{\tilde{a}}(\xvec\mid B_i^\mathcal{A}) &&(\because \xvec_i^\mathcal{A} \text{~is an $\alpha$-approximation})
                \\
                &\ge \frac{\alpha}{|W|}\sum_{\xvec\in W} \Delta_{\tilde{a}}(\xvec\mid B_i^\mathcal{A})  \\
                &\ge \frac{\alpha\hat{\gamma}}{2|W|}(s(B_n^*)-s(B_n^\mathcal{A})) + \frac{i\alpha \lambda}{n(n-1)}\mathrm{div}(B_n^*) && (\because \text{\Cref{eq:modivprev}})\\
                &\ge \frac{\alpha\hat{\gamma}}{2n}(s(B_n^*)-s(B_n^\mathcal{A})) + \frac{i\alpha \lambda}{n(n-1)}\mathrm{div}(B_n^*). && (\because |W|\le |B_n^*| =n)
            \end{align*} 
By summing the inequality above for all $i$ from $0$ to $n-1$, we have
                \begin{align*}
                \frac{1}{2}s(B_n^\mathcal{A}) + \lambda \mathrm{div}(B_n^\mathcal{A}) = \tilde{a}(B_n^\mathcal{A}) &= \sum_{i=0}^{n-1} \Delta_{\tilde{a}}(\xvec_{i}^\mathcal{A}\mid B_i^\mathcal{A}) \\
                &\ge \frac{\alpha\hat{\gamma}}{2}(s(B^*_n) - s(B^\mathcal{A}_n)) + \frac{\alpha\lambda}{2}\mathrm{div}(B^*_n)\\
                &\ge \frac{\alpha \hat{\gamma}}{2}(s(B_n^*) - s(B_n^\mathcal{A}) + \lambda \mathrm{div}(B_n^*)). &&(\because \hat{\gamma}\le 1 \text{~and non-negativity of }\mathrm{div} ) 
            \end{align*}
Hence, $$\frac{1+\alpha\hat{\gamma}}{2}s(B_n^\mathcal{A}) + \lambda \mathrm{div}(B_n^\mathcal{A}) \ge \frac{\alpha\hat{\gamma}}{2}(s(B_n^*) + \lambda \mathrm{div}(B_n^*)).$$
Finally, we have
                \begin{align*}
                a(B_n^\mathcal{A}) = s(B_n^\mathcal{A}) + \lambda \mathrm{div}(B_n^\mathcal{A})&\ge \frac{1+\hat{\gamma}\alpha}{2}s(B_n^\mathcal{A}) + \lambda \mathrm{div}(B^\mathcal{A}_n) && (\because \hat{\gamma}, \alpha \le 1)\\
                &\ge \frac{\alpha\hat{\gamma}}{2}(s(B_n^*) + \lambda \mathrm{div}(B_n^*)) = \frac{\alpha\hat{\gamma}}{2} a(B^*_n).
            \end{align*}
        \qed

\subsubsection{Connection to Prior Bounds}
\begin{table}[hbt!]
    \centering
    \caption{Bounds for the exact greedy algorithm and the approximated greedy algorithm.}
    \resizebox{\columnwidth}{!}{
    \begin{tabular}{cllll}
    \toprule
    &  \multicolumn{2}{l}{Exact Greedy Algorithm} & \multicolumn{2}{l}{Approximated Greedy Algorithm}\\
    \cmidrule(lr){2-3}
    \cmidrule(lr){4-5}
         Condition on Acquisition & w/o diversity & w/ diversity & w/o diversity & w/ diversity\\
         \midrule
        Submodular & $1 - 1/e$ \citep{exactgreedybound} & $1/2$ \citep{borodin} & $1 - (1/e)^{\alpha}$ \citep{approxgreedy} & $\alpha / 2$ (\Cref{thm:divmodbound}, $\gamma=1$)\\
        Near-submodular & $1 - (1/e)^\gamma$ \citep{submodindexbound} & $\gamma/2$ (\Cref{thm:divmodbound}, $\alpha=1$) & $1 - (1/e)^{\alpha\gamma}$ (\Cref{thm:modbound})	&$\alpha\gamma / 2$ (\Cref{thm:divmodbound})\\
        \bottomrule
    \end{tabular}}
    \label{tab:boundsummary}
\end{table}
Since $\gamma_{B,n}(a) = 1$ for any $B$, $n$ when $a$ is monotone submodular, \Cref{thm:modbound} directly contains a bound for the monotone submodular case (\Cref{cor:submodbound}) proved by \citet{approxgreedy}.
\begin{corollary}
\label{cor:submodbound}
    Let $a:\mathcal{X}\to\reals$ be a non-negative monotone submodular set function. If $\mathcal{A}$ is an $\alpha$-approximation algorithm, \Cref{alg:approxgreedy} returns an $(1-1/e^\alpha)$-approximation to the optimal $n$-subset.
\end{corollary} 
Similarly, \Cref{thm:divmodbound} directly contains the following corollary.
\begin{corollary}
\label{cor:divsubmodbound}
    Let $s$ be a non-negative monotone submodular set function and $\mathrm{div}$ be a sum-dispersion function. If $\mathcal{A}$ is an $\alpha$-approximation algorithm, 
    \Cref{alg:approxgreedy} with the set function $(s/2+ \lambda \mathrm{div})$ returns $B_n$, an $(\alpha/2)$-approximation to the optimal $n$-subset of $(s+\lambda \mathrm{div})$.
\end{corollary}
Note that the $\alpha=1$ case of \Cref{cor:divsubmodbound} is the same as the bound for the exact greedy algorithm proved by \citet{borodin}.
Finally, \Cref{tab:boundsummary} summarizes the prior bounds and new bounds for the exact greedy algorithm and the approximated greedy algorithm for various conditions on the set function.

\newpage
\twocolumn
\section{Implementation Details}
\subsection{MDP Designs}
This paper considers the benchmark tasks on sequence data such as proteins and aptamers. 
There are a variety of MDP designs to model the sequence data as MDPs \citep{shen}. 
In this paper, we consider two designs, \emph{appending MDP} and \emph{editing MDP} of MDPs following prior works \citep{jain2023multi}. 
First, appending MDP designs any sequence from scratch. Assume a simple scenario that a search space is given by the fixed length sequence space $\mathcal{X}=\mathcal{V}^L$ is given by its length $L$ and a vocabulary set $\mathcal{V}$. For this space, we can write the state space and the action space as follows:
$$\mathcal{S} = \bigcup_{i=0}^L \mathcal{V}^L \cup \{s_\mathrm{term}\},\mathcal{A} = \mathcal{V} \cup \{a_\mathrm{term}\}.$$
Here, the initial state is always given as an empty sequence, and the MDP has a deterministic transition function $\mathcal{T}:\mathcal{S}\times\mathcal{A}\to\mathcal{S}$. Concretely,
$$\mathcal{T}(s,a)=\begin{cases}
  s \oplus a & \text{if $a\in\mathcal{V}$ and $\mathrm{len}(s)\le L-1$ },\\
  s_\mathrm{term} &\text{otherwise}.
\end{cases}$$
In other words, action $a\in\mathcal{A}$ appends the chosen token to the current state $s\leftarrow s\oplus a$ or terminate the construction $s\leftarrow s_\mathrm{term}$. As introduced in \Cref{main:rlintro}, a reward is given at the terminal state by the objective value of a resulting sequence. Hence, we do not distinguish the terms rewards and returns in this scenario.
In this MDP design, a generic policy model $\pi_\theta:\mathcal{S}\to\mathcal{A}$ outputs an action distribution (token distribution) given state $s$ (subsequence).

Otherwise, editing MDP designs any sequence by editing a sequence in the given pool of candidates. 
Hence, an action corresponds to an edit operation on the sequence. 
As in \citet{jain2023multi}, we concentrate on substitution operations for the action space. 
In short, the state space is a possible set of sequences obtained by editing a given pool of candidates. The action space can be represented by 
$$\mathcal{A} = ([L_\text{max}] \times \mathcal{V}) \cup \{a_\mathrm{term}\}$$ where $L_\text{max}$ is the maximum length of the sequence that MDP considers, and $\mathcal{V}$ is a vocabulary set. Briefly, an action $(l, t)\in [L_\text{max}]\times \mathcal{V}$ substitutes $l$-th token in the target sequence to $t$.
 
\subsection{Architectures}
\label{sec:arch}
\textbf{State encoders and action decoders.}
We adopt the policy model architectures described by \citet{jain2023multi} for encoding states and decoding action logits in both appending and editing MDPs. The encoder architectures for both types of MDPs leverage transformer architectures to convert sequences into hidden features \citet{Devlin2019BERTPO}. In the appending MDP scenario, an MLP head predicts the action logit corresponding to the next token to be appended. For editing MDPs, an additional MLP head is employed to predict logits on positions, indicating the probability of substituting at that position. For a fair comparison, we maintain consistent configurations for state encoders and action decoders across all MDP-based subset selection methods (PC-RL, PC-MOGFN, Greedy + RL, and Ours).

\textbf{Set encoders.}
Like preference conditioning methods, we incorporate a set encoder to extract features from given sets, employing a deep set architecture designed for point cloud classification tasks \citep{deepset}. This architecture includes $3$ equivariant max-pooling layers with tanh activations and a $2$-layer MLP head to derive hidden features from the sets. The parameter count in the set encoder depends on the dimension $m$ of each point in the set and the hidden feature dimension $N_\text{hid}$, where $m$ aligns with the number of objectives in our approach. Note that the set encoder's total parameter count, calculated as $2mN_\text{hid} + 6N_\text{hid}^2$, is substantially lower than that of transformer-based state encoders. For instance, a 3-layer transformer employed in bigrams tasks possesses more than $30N_\text{hid}^2$ parameters.

\textbf{Incorporating MLM logits for decoding actions.}
Following \citet{jain2023multi}, we utilize appending MDP for single-round subset selection tasks with deterministic objective functions. 
For batch BO experiments in biological sequences, we utilize editing MDP architectures. Drawing inspiration from the LaMBO architecture \citep{lambo}, we introduce a variant of our method that integrates an MLM model trained on previously evaluated data points during action decoding. Like LaMBO, which optimizes in MLM latent space using data-trained autoencoders, our approach benefits from initiating optimization with tokens likely found in the evaluated data. 
Note that, in our experimental setup, the MLM model is concurrently trained with the surrogate model. Additionally, for the pool of candidates to be edited, we employ hashing on the MLM predictions for each position, introducing a minimal additional computational cost.

\subsection{Experimental Settings}

\label{sec:settings}
\textbf{Single-round subset selection with Bigrams tasks.}
For our main experiments in single-round subset selection, we employ the bigrams benchmark tasks as implemented by \citet{lambo} and \citet{jain2023multi}. Each task is designed around target bigrams, with each corresponding to a specific objective. In this setup, synthetic bigram matching objectives serve as a deterministic surrogate model for tackling the active learning inner loop problem, with no prior data points evaluated and the Hypervolume indicator functioning as the batch acquisition function. \Cref{tab:exp:bigrams} details the bigrams tasks.

We compare our method against PC-RLs and greedy-based approaches.
Note that we train a single set-conditioned policy model to sample subsets with various cardinalities in this experiment.
Our set-conditioned policy is trained with parameters $n_\text{train}=64$, $N_t=4$, over $N_u=4000$ update steps, and a batch size of $N_e= 128$. At every 500 steps, we perform greedy sampling $\mathrm{GS}(a,\pi_\theta^\text{set},n,l=128)$ for each cardinality constraint $n$ to evaluate the Hypervolume indicator value. After the training ends, we report the best evaluated results for each $n$.

Given that $n_\text{train}/2$ samples are used for subset $B$ conditioning at each step, our method uses $N_e + n_\text{train}/2 = 128 + 32 = 160$ samples per step. In contrast, other RL-based methods using $N_e=128$ operate with 20\% fewer samples. For a fair comparison, we provide other RL-based methods, including PC-RL (TS), PC-RL (WS), and Greedy + RL, with total $5000$ update steps, 25\% more steps than ours.

For the PC-RL baselines, preference vectors are sampled from a Dirichlet distribution with $\alpha=1$. Like our method, we train a single PC-policy and utilize this policy to sample subsets in various cardinalities $n$s.
To construct $n$-subset, we sample $n$ preference vectors. 
For each vector, $l=128$ candidates are sampled, conditioned on the preference vector. Then, the $n$-subset is formed by choosing the top candidate for each preference vector, adhering to protocols established in PC-based methods \citep{jain2023multi,zhu2023sampleefficient}.
Similar to our approach,
this sampling process is executed for each cardinality $n$ at every 500 steps, continuing until the total number of update steps, $N_u=5000$, is attained.
\begin{table}[hbt!]
    \centering
    \caption{Settings of bigrams tasks.}
    \resizebox{\columnwidth}{!}{
    \begin{tabular}{crrrr}
\toprule
    Task & Target Bigrams & Min. Len. & Max. Len.& Cardinalities\\
    \midrule
       2 Bigrams &  AV, VC & 32 & 36 & 4, 16\\
       3 Bigrams &  AV, VC, CA& 32 & 36 & 4, 16, 64\\
       4 Bigrams &  AV, VC, CA, AW& 32 & 36 & 4, 16, 64, 256\\
       \bottomrule
    \end{tabular}}
    \label{tab:exp:bigrams}
\end{table}
\begin{table}[hbt!]
\centering
\caption{Hyperparameters for bigrams tasks and the DNA aptamer task}
\begin{tabular}{rr}
\toprule
Hyperparameter & Values\\
   \midrule
   $\eta$   & 1E-4, 1E-5, 1E-6 \\
   Random Action Prob.  & 0, 0.05\\
   \bottomrule
\end{tabular}
\label{tab:hyperparams}
\end{table}
For greedy approaches, optimization is carried out individually for each $n$, with a total budget of $B = N_u * N_e + (N_u/500) \times n \times l$ allocated for surrogate model queries across the process for a fair comparison.  This allocation allows each iteration of the greedy method to use $B/n$ budget for optimization.  Greedy + RS selects the best sequence from $B/n$ randomly sampled sequences at each iteration. Greedy + HC starts from a random sequence and iteratively moves to the optimal sequence within a 1-Hamming distance, restarting from another random point if necessary until the surrogate model budget of $B/n$ is reached.
Specifically, Greedy + RL utilizes $N_e=128$ and sets $N_u/n$ as update steps for each greedy loop. We determine the number of samples to deploy during each greedy loop based on $B/n$. 

All RL-based methods employ a transformer encoder architecture with 3 layers, 8 heads, and a hidden dimension of 128. 
Also, we normalize returns as in \Cref{alg:trainsetpolicy} for all RL-based methods.
For RL-based methods (Ours, PC-RLs, Greedy + RL), we  tune the hyperparameters among the combinations in \Cref{tab:hyperparams}. For each combination of hyperparameters, we run 10 trials and  report the result from the best hyperparameters.

\textbf{Single-round subset selection with DNA aptamers.}
We utilize three objectives, the number of hairpins, the number of pairs, and the energy value computed by the NUPACK library \citep{nupack}, adopting the implementation of \citet{jain2023multi}.
In this setting, we use a larger transformer architecture with 4 layers, 16 heads, and a hidden dimension of 256. 
We set $N_u=2000$ for our method and allocate 25\% more update steps to other methods as in bigrams tasks.
Other parameter settings are identical to the bigrams tasks.

\begin{table*}[hbt!]
    \centering
    \caption{Ablation on training cardinality constraint. The mean and standard deviation values are calculated for 10 trials.}
\resizebox{\textwidth}{!}{
\begin{tabular}{lccccccccccccccccccc}
\toprule
& \multicolumn{9}{c}{Hypervolume Indicator ($\uparrow$)}\\
\cmidrule(lr){2-10}
         & \multicolumn{2}{c}{2 Bigrams}& \multicolumn{3}{c}{3 Bigrams}& \multicolumn{4}{c}{4 Bigrams}\\
    \cmidrule(lr){2-3}
    \cmidrule(lr){4-6}
    \cmidrule(lr){7-10}
Method     &  $n=4$ & $n=16$ &  $n=4$ & $n=16$ & $n=64$& $n=4$ & $n=16$ & $n=64$ & $n=256$\\
    \cmidrule(lr){1-3}
    \cmidrule(lr){4-6}
    \cmidrule(lr){7-10}
Optimum & \multicolumn{2}{c}{0.630} & \multicolumn{3}{c}{0.409}  & \multicolumn{4}{c}{0.106} \\
    \cmidrule(lr){1-3}
    \cmidrule(lr){4-6}
    \cmidrule(lr){7-10}
Exact Greedy& 0.568 & 0.630 & 0.350 & 0.408 & 0.409 & 0.055 & 0.078 & 0.097 & 0.106   \\
    \cmidrule(lr){1-3}
    \cmidrule(lr){4-6}
    \cmidrule(lr){7-10}
Ours ($n_\text{train}=128$) &\textbf{0.568} (0.000)&\textbf{0.630} (0.000)&\textbf{0.329} (0.005)&0.345 (0.003)&0.354 (0.006)& \textbf{0.055} (0.001)& 0.076 (0.001)&0.090 (0.002)&\textbf{0.094} (0.002)\\
Ours ($n_\text{train}=64$)       & \textbf{0.568} (0.000)& \textbf{0.630} (0.000) & \textbf{0.329} (0.005) & \textbf{0.349} (0.007) & \textbf{0.359} (0.003)&\textbf{0.055} (0.000) &\textbf{0.077} (0.000)& \textbf{0.091} (0.002)& \textbf{0.094} (0.003) \\
Ours ($n_\text{train}=32$) &\textbf{0.568} (0.000)&\textbf{0.630} (0.000)&0.328 (0.003)&0.345 (0.007)&\textbf{0.359} (0.005)&\textbf{0.055} (0.000)&0.076 (0.000)& 0.086 (0.002)&0.089 (0.002)\\
Ours ($n_\text{train}=16$)  &0.564 (0.013)&0.622 (0.023)&0.326 (0.006)&0.352 (0.005)&0.355 (0.005)&\textbf{0.055} (0.000) &0.074 (0.001)&0.081 (0.002)&0.084 (0.001)\\
Ours ($n_\text{train}=4$)  &0.525 (0.000)&0.537 (0.000)&0.326 (0.007)&0.347 (0.007)&0.355 (0.005)&0.052 (0.001)&0.065 (0.002)&0.069 (0.003)&0.072 (0.003)\\
\bottomrule
\end{tabular}}
\label{tab:appmain}
\end{table*}

\textbf{Batch BO Experiments.}
For the batch BO, we consider three benchmarks from \citet{lambo}. 
Our primary benchmark is the RFP task, which optimizes the stability and solvent-accessible surface area (SASA) of RFPs. Additionally, we conduct experiments on two other benchmarks: 3 Bigrams (\Cref{tab:exp:bigrams}) and small molecules. The latter optimizes the logP and quantitative estimate of drug-likeness (QED) of SELFIES-encoded small molecules \citep{qed,selfies}.

Beyond addressing the issues identified in LaMBO's implementation (as detailed in \Cref{sec:issues}), we adopt similar experimental setups, with the exception of the number of samples generated at each inner loop. 
We train a set-conditioned policy with $n_\text{train}=n=16$, $N_t=1$, over $N_u=256$ update steps and a batch size of $N_e=128$. At every $64$ step, we perform greedy sampling $\mathrm{GS}(a,\pi_\theta^\text{set},n=16,l=16)$, and propose the best sampled subset with the highest batch acquisition value for the proposal batch.
To ensure a fair comparison, we increase the number of samples generated per step of each baseline method, from $32$ to $2048$ for MBGA, and from $16$ to $256$ for LaMBO.
These modification leads to improvement in performance of baseline methods as illustrated in \Cref{fig:raise}.
In addition, the modification results in the end-to-end process for the RFP task taking a similar scale of runtime between $2$ to $3$ days for all active learning based methods we consider in this scenario. 
Also, we utilize the same architecture and training algorithm for updating MTGP models for a fair comparison.
For set conditioning, we set the continuous features, $\mathrm{feat}(\xvec)\coloneqq \tilde{f}_\mathrm{UCB}(\xvec;\beta=0.1)$, for all experiments with statistical surrogate models.
Also, we set the learning rate $\eta=0.0001$ and set the random action probability to $0$.
Finally, we set the maximum edit budget of editing MDP to $1$ as in LaMBO and MBGA.

\section{Additional Experiments}
\input{figures/fig_regex}
\subsection{Additional Results on Synthetic Tasks}
\label{sec:addexp1}
\textbf{Ablation on cardinality constraint.}
In our experiments, we differentiate training cardinality constraints from sampling constraints during our experiments on bigrams tasks. We adjust the training cardinality ($n_\text{train}$) and evaluate the effectiveness of models across varying set sizes through greedy sampling. To ensure comparable execution speed, $N_t$, the training step count, is set to  $\max(1, n_\text{train}/16)$. \Cref{tab:appmain} demonstrates that models trained with larger set sizes yield superior results across both smaller and larger cardinalities in bigrams tasks. 
As we introduced in \Cref{sec:settings}, `Ours' with training cardinality constraint $n_\text{train}=64$ corresponds to the `Ours' in \Cref{tab:main}.

\subsection{Additional Results on Batch BO Benchmarks}
\label{sec:addexp2}
\textbf{Additional multi-round batch BO results.}
\Cref{fig:bigramsapp} illustrates the multi-round batch BO results on the 3 bigrams task with the MTGP surrogate model and NEHVI batch acquisition function. The results show that our method without MLM achieves higher relative Hypervolume indicator values faster than the baseline active learning results. However, for this synthetic task, MLM based strategy was not helpful for achieving the better performance as reported in \citet{lambo}.
Next, we conduct multi-round batch BO with UCBHVI batch acquisition function on the RFP task. 
\Cref{fig:RFPUCB} illustrates the performance of our method and the baseline methods equipped with UCBHVI.
The results show that our methods achieve superior performance in this setting.

\input{figures/fig_app_rfp_ucb}
\begin{table}[t]
    \centering
    \caption{Subset selection results in the first round when optimizing various batch acquisition functions (NEHVI, UCBHVI, PES) on the RFP task. `Ours-half' refers to our method with half the number of update steps. The mean and standard deviation values are calculated for 10 trials.}
    \begin{subtable}[t]{\columnwidth}
    \centering
    \caption{NEHVI.}
    \vspace{-0.5em}
    \resizebox{0.8 \columnwidth}{!}{
    \begin{tabular}{ccccc}
    \toprule
Method & NEHVI Value ($\uparrow$) & Runtime (mins) ($\downarrow$)\\
    \cmidrule(lr){1-3}
Ours	&\textbf{0.779 (0.045)}	&18.9 \\
Ours-half	&0.778 (0.033)	&\textbf{9.5} \\
LaMBO	&0.591 (0.033)	&24.4 \\
MBGA	&0.654 (0.052)	&14.0 \\
\bottomrule
    \end{tabular}}
    \end{subtable}
    
    \vspace{0.5em}
    \begin{subtable}[t]{\columnwidth}
    \centering
    \caption{UCBHVI.}
    \vspace{-0.5em}
    \resizebox{0.8\columnwidth}{!}{
    \begin{tabular}{ccccc}
    \toprule
Method & UCBHVI Value ($\uparrow$) & Runtime (mins) ($\downarrow$)\\
    \cmidrule(lr){1-3}
Ours	&\textbf{1.019 (0.032)}&	4.1\\
Ours-half	&1.005 (0.034)&	\textbf{2.1}\\
LaMBO	& 0.776 (0.022)	&16.7\\
MBGA   & 0.844 (0.054)	&9.4\\
\bottomrule
    \end{tabular}}
    \end{subtable}
    
    \vspace{0.5em}
    \begin{subtable}[t]{\columnwidth}
    \centering
    \caption{PES.}
    \vspace{-0.5em}
    \resizebox{0.8\columnwidth}{!}{
    \begin{tabular}{ccccc}
    \toprule
Method & PES Value ($\uparrow$) & Runtime (mins) ($\downarrow$)\\
    \cmidrule(lr){1-3}
Ours	&\textbf{1.233 (0.040)}&	\textbf{12.8}\\
LaMBO w/ FD	& 0.070 (0.015)	&18.0\\
MBGA   & 0.207 (0.107)	&29.8\\
\bottomrule
    \end{tabular}}
    \end{subtable}
    \label{tab:statucbhvi}
\end{table}

\textbf{Additional subset selection results in the first round.}
To demonstrate the scalability and broad applicability of our method, we additionally evaluate the first-round subset selection performance, including runtime, across various batch acquisition functions. 
\Cref{tab:statucbhvi} presents the acquisition values and runtime obtained by optimizing NEHVI, UCBHVI, and PES in the RFP task. 
For PES, we use the implementation in the \emph{BoTorch} framwork\footnote{\href{https://botorch.org/tutorials/information_theoretic_acquisition_functions}{https://botorch.org/tutorials/information\_theoretic\_acquisition\\\_functions}.} 
and we modify LaMBO to use a gradient approximated by finite differences (FD) due to the non-differentiability of PES computation \citep{botorch}. 
Our method (`Ours-half' for NEHVI and `Ours' for UCBHVI and PES) achieved higher batch acquisition values in less runtime compared to baseline methods, demonstrating the effectiveness and scalability of our approach when optimizing various batch acquisition functions with statistical surrogate models.

\subsection{Diversified subset selection results}
\input{figures/fig_div}

\Cref{fig:div} illustrates the results of diversified subset selection for 
2 bigrams tasks, comparing our method with PC-MOGFN. 
The results show that our method succeed to generate diverse candidates while keeping ability to generate near optimal solutions in the 2 bigrams task.
\Cref{tab:div} provides a summary of the diversified subset selection results in the first round on the RFP task in comparison with AL-MOGFN. The findings indicate that our method is capable of constructing subsets that are more diverse and have higher NEHVI values than those generated by the baseline method. However, unlike HVI-based batch acquisition functions, the features used for set conditioning in our method, $\mathrm{feat}(\xvec)=\tilde{f}(\xvec)$, might not offer enough information to steer the policy towards generating a variety of candidates. The development of techniques for extracting features that enhance diversity remains an area for future research in our study.

\begin{table}[t]
    \centering
    \caption{Diversified subset selection results in the first round of the RFP task with NEHVI. The mean and standard deviation values are calculated for 10 trials.}
    \resizebox{\columnwidth}{!}{
    \begin{tabular}{ccccc}
    \toprule
Method & NEHVI Value ($\uparrow$) & Diversity ($\uparrow$)\\
\midrule
Ours w/o MLM ($\lambda=0.0$) & 0.779 (0.045) & 78.352 (6.194)\\
\midrule
Ours w/o MLM ($\lambda=1.0$) & \textbf{0.731} (0.033) & \textbf{95.917} (0.35)\\
AL-MOGFN ($\beta=16$)  &  0.608 (0.074) & 93.253 (0.346)\\
AL-MOGFN ($\beta=24$)  &  0.613 (0.061) & 93.240 (0.200)\\
\bottomrule
    \end{tabular}}
    \label{tab:div}
\end{table}

\section{Addressing Previous Issues in LaMBO}
\label{sec:issues}
\input{figures/fig_multi}
The work by \citet{lambo} has made significant contributions to establishing benchmarks for biological sequence design. Nonetheless, certain challenges were identified in the LaMBO implementation that impacted its performance. Firstly, the original implementation wrongly calculated the NEHVI batch acquisition values for batches containing more than one element. Secondly, an error in the mutation operation used in the GA methods was discovered, adversely affecting performance across several tasks.

In our study, we rectify these issues and conduct a performance comparison between the original version, our corrected version, and an enhanced version with a larger sample size, which we use as the baselines in our paper for a fair comparison. \Cref{fig:raise} presents the benchmark results for the RFP task and 3 bigrams task. Notably, correcting these issues led to a substantial improvement in performance on the 3 bigrams task, achieving more than 3 times larger relative Hypervolume compared to the original implementation. Additionally, it was observed that increasing the number of samples during the inner loop contributed to improved performance for these tasks.

\end{document}

%% file: def.tex
\newcommand{\maximize}{\operatorname*{maximize}}

\newcommand{\reals}{\mathbb R}

\newcommand{\argmax}{\mathop{\rm argmax}}

\newcommand{\eg}{{\it e.g.}}
\newcommand{\ie}{{\it i.e.}}

%% file: figures/fig_main.tex
\begin{figure}[t]
\centering
\begin{subfigure}{0.45\textwidth}
	\centering
\resizebox{\textwidth}{!}{
		\begin{tikzpicture}
		\begin{axis}[
		width=8.5cm,
		height=8.0cm,
		every axis plot/.append style={thick},
		grid=major,
		scaled ticks = false,
		ylabel near ticks,
		tick pos=left,
		tick label style={},
		xtick={0, 250, 500, 750, 1000},
		xticklabels={0, 250, 500, 750, 1000},
		ytick={0, 0.5, 1.0, 1.5, 2.0, 2.5, 3.0},
		yticklabels={0, 0.5, 1.0, 1.5, 2.0, 2.5, 3.0},
		label style={},
		xlabel={Number of queries},
		xlabel style={at={(0.5,0.0)}, font=\large},
		ylabel={Relative Hypervolume ($\uparrow$)},
		ylabel style={align=center, at={(-0.1,0.5)}, font=\large},
		xmin=0,
		xmax=1024,
		ymin=0.9,
		ymax=3.2,
            legend cell align={left},
		legend style={legend columns=1, at={(1.6, 0.77)}, font=\small},
        ]
        \begin{pgfonlayer}{main}
            \addplot[cyan, mark size=1.5pt, no markers] table [x=x_vals, y=y_med, col sep=comma]{csvs/main_proxy_rfp/Ours_0.0001_use_mlm.csv};
		\addlegendentry{Ours w/ MLM}
            \addplot[blue, mark size=1.5pt, no markers] table [x=x_vals, y=y_med, col sep=comma]{csvs/main_proxy_rfp/Ours_0.0001.csv};
		\addlegendentry{Ours w/o MLM}
            \addplot[black, dotted] coordinates {(0,1.96) (1024,1.96)};
		\addlegendentry{MOGFN (Original)}
            \addplot[green!80!black!100, mark size=1.5pt, no markers] table [x=x_vals, y=y_med, col sep=comma]{csvs/main_proxy_rfp/LaMBO_Long.csv};
		\addlegendentry{LaMBO (Reprod.)}
  		\addplot[green!80!black!60, mark size=1pt, dashed, no markers] table [x=x_vals, y=y_med, col sep=comma]{csvs/github_proxy_rfp/LaMBO_Github.csv};
		\addlegendentry{LaMBO (Original)}
    \addplot[red, mark size=1.5pt, no markers] table [x=x_vals, y=y_med, col sep=comma]{csvs/main_proxy_rfp/MBGA_Long.csv};
		\addlegendentry{MBGA (Reprod.)}
  \addplot[red!50, mark size=1pt, dashed, no markers] table [x=x_vals, y=y_med, col sep=comma]{csvs/github_proxy_rfp/MBGA_Github.csv};
		\addlegendentry{MBGA (Original)}
  		\addplot[magenta!70, mark size=1.5pt, no markers] table [x=x_vals, y=y_med, col sep=comma]{csvs/main_proxy_rfp/NSGA.csv};
		\addlegendentry{NSGA-II}
        \end{pgfonlayer}
            \addplot[name path=n1, black, mark size=1.5pt, no markers, line width=0pt,opacity=0] table [x=x_vals, y=y_ub, col sep=comma]{csvs/main_proxy_rfp/Ours_0.0001_use_mlm.csv};
            \addplot[name path=n2, black, mark size=1.5pt, no markers, line width=0pt,opacity=0] table [x=x_vals, y=y_lb, col sep=comma]{csvs/main_proxy_rfp/Ours_0.0001_use_mlm.csv};
            \addplot[name path=r1,red!20, mark size=1.5pt, no markers, line width=0pt,opacity=0] table [x=x_vals, y=y_ub, col sep=comma]{csvs/main_proxy_rfp/Ours_0.0001.csv};
            \addplot[name path=r2,red!20, mark size=1.5pt, no markers, line width=0pt,opacity=0] table [x=x_vals, y=y_lb, col sep=comma]{csvs/main_proxy_rfp/Ours_0.0001.csv};
            \addplot[name path=b1,blue!20, mark size=1.5pt, no markers, line width=0pt,opacity=0] table [x=x_vals, y=y_ub, col sep=comma]{csvs/main_proxy_rfp/LaMBO_Long.csv};
            \addplot[name path=b2,blue!20, mark size=1.5pt, no markers, line width=0pt,opacity=0] table [x=x_vals, y=y_lb, col sep=comma]{csvs/main_proxy_rfp/LaMBO_Long.csv};
    \addplot[name path=g1,green!20, mark size=1.5pt, no markers, line width=0pt,opacity=0] table [x=x_vals, y=y_ub, col sep=comma]{csvs/main_proxy_rfp/MBGA_Long.csv};
    \addplot[name path=g2,green!20, mark size=1.5pt, no markers, line width=0pt,opacity=0] table [x=x_vals, y=y_lb, col sep=comma]{csvs/main_proxy_rfp/MBGA_Long.csv};

  		\addplot[name path=m1,magenta!15, mark size=1.5pt, no markers, line width=0pt,opacity=0] table [x=x_vals, y=y_ub, col sep=comma]{csvs/main_proxy_rfp/NSGA.csv};
  		\addplot[name path=m2,magenta!15, mark size=1.5pt, no markers, line width=0pt,opacity=0] table [x=x_vals, y=y_lb, col sep=comma]{csvs/main_proxy_rfp/NSGA.csv};
                  \addplot fill between[ 
            of = r1 and r2, 
            split, 
            every even segment/.style = {blue!30,opacity=0.5},
            ];
                      \addplot fill between[ 
            of = b1 and b2, 
            split, 
            every even segment/.style = {green!50,opacity=0.5},
            ];
                  \addplot fill between[ 
            of = g1 and g2, 
            split, 
            every even segment/.style = {red!30,opacity=0.5},
            ];
                      \addplot fill between[ 
            of = m1 and m2, 
            split, 
            every even segment/.style = {magenta!30,opacity=0.5},
            ];
                      \addplot fill between[ 
            of = n1 and n2, 
            split, 
            every even segment/.style = {cyan!30,opacity=0.5},
            ];
		\end{axis}
		\end{tikzpicture}
}
 \vspace{-1.6em}
 \subcaption{Multi-round active learning results.}
	\label{fig:multiround}
\end{subfigure}
 \vspace{0.8em}
\begin{subfigure}{0.45\textwidth}
	\centering
\resizebox{\textwidth}{!}{
		\begin{tikzpicture}
		\begin{axis}[
		width=8.5cm,
		height=8.0cm,
		every axis plot/.append style={thick},
		grid=major,
		scaled ticks = false,
		ylabel near ticks,
		tick pos=left,
		tick label style={},
		xtick={10000,11000, 12000, 13000, 14000},
		xticklabels={10000,11000, 12000, 13000, 14000},
		ytick={-20, 0, 20, 40, 60, 80},
		yticklabels={-20, 0, 20, 40, 60, 80},
		label style={},
		xlabel={SASA ($\uparrow$)},
		xlabel style={at={(0.5,0.0)}, font=\large},
		ylabel={Stability ($\uparrow$)},
		ylabel style={align=center, at={(-0.1,0.5)}, font=\large},
		xmin=9900,
		xmax=14100,
		ymin=-40,
		ymax=100,
            legend cell align={left},
		legend style={legend columns=1, at={(1.6, 0.89)}, font=\small},
        ]
            \addplot[black, no marks] table [x=xval, y=yval, col sep=comma]{csvs/pareto_proxy_rfp/ours_total.csv};
            \addlegendentry{Frontier by Ours}
            \addplot[black, mark=+, mark size=2pt, only marks] coordinates {
                (-100,-10000)
            };
            \addlegendentry{Offsprings by Ours}
            \addplot[black, no marks, dotted] table [x=xval, y=yval, col sep=comma]{csvs/pareto_proxy_rfp/mogfn_total.csv};
		\addlegendentry{Frontier by MOGFN}

            \addplot[black, mark=x, mark size=2pt, only marks] coordinates {
                (-100,-10000)
            };
            \addlegendentry{Offsprings by MOGFN}

            \addplot[black, no marks, dashed] table [x=xval, y=yval, col sep=comma]{csvs/pareto_proxy_rfp/old_total.csv};
		\addlegendentry{Initial RFPs}
            \addplot[blue, only marks, mark=o, mark size=2pt, mark options={line width=1.5pt}] table [x=xval, y=yval, col sep=comma]{csvs/pareto_proxy_rfp/old_DsRed.M1.csv};
            \addlegendentry{DsRed.M1}
            \addplot[orange, only marks, mark=o, mark size=2pt, mark options={line width=1.5pt}] table [x=xval, y=yval, col sep=comma]{csvs/pareto_proxy_rfp/old_mScarlet.csv};
            \addlegendentry{mScarlet}
            \addplot[green!80!black!100, only marks, mark=o, mark size=2pt, mark options={line width=1.5pt}] table [x=xval, y=yval, col sep=comma]{csvs/pareto_proxy_rfp/old_DsRed.T4.csv};
            \addlegendentry{DsRed.T4}
            \addplot[red, only marks, mark=o, mark size=2pt, mark options={line width=1.5pt}] table [x=xval, y=yval, col sep=comma]{csvs/pareto_proxy_rfp/old_AdRed.csv};
            \addlegendentry{AdRed}
            \addplot[purple!80!black!100, only marks, mark=o, mark size=2pt, mark options={line width=1.5pt}] table [x=xval, y=yval, col sep=comma]{csvs/pareto_proxy_rfp/old_mRouge.csv};
            \addlegendentry{mRouge}
            \addplot[brown, only marks, mark=o, mark size=2pt, mark options={line width=1.5pt}] table [x=xval, y=yval, col sep=comma]{csvs/pareto_proxy_rfp/old_RFP630.csv};
            \addlegendentry{RFP630}

            \addplot[blue, mark=x, mark size=2pt, only marks] table [x=xval, y=yval, col sep=comma]{csvs/pareto_proxy_rfp/mogfn_DsRed.M1.csv};
            \addplot[orange, mark=x, mark size=2pt, only marks] table [x=xval, y=yval, col sep=comma]{csvs/pareto_proxy_rfp/mogfn_mScarlet.csv};
            \addplot[green!80!black!100, mark=x, mark size=2pt, only marks] table [x=xval, y=yval, col sep=comma]{csvs/pareto_proxy_rfp/mogfn_DsRed.T4.csv};
            \addplot[purple!80!black!100, mark=x, mark size=2pt, only marks] table [x=xval, y=yval, col sep=comma]{csvs/pareto_proxy_rfp/mogfn_mRouge.csv};

            \addplot[blue, mark=+, mark size=1pt, only marks] table [x=xval, y=yval, col sep=comma]{csvs/pareto_proxy_rfp/ours_DsRed.M1.csv};
            \addplot[orange, mark=+, mark size=1pt, only marks] table [x=xval, y=yval, col sep=comma]{csvs/pareto_proxy_rfp/ours_mScarlet.csv};
            \addplot[green!80!black!100, mark=+, mark size=1pt, only marks] table [x=xval, y=yval, col sep=comma]{csvs/pareto_proxy_rfp/ours_DsRed.T4.csv};
            \addplot[red, mark=+, mark size=1pt, only marks] table [x=xval, y=yval, col sep=comma]{csvs/pareto_proxy_rfp/ours_AdRed.csv};
            \addplot[purple!80!black!100, mark=+, mark size=1pt, only marks] table [x=xval, y=yval, col sep=comma]{csvs/pareto_proxy_rfp/ours_mRouge.csv};
            \addplot[brown, mark=+, mark size=1pt, only marks] table [x=xval, y=yval, col sep=comma]{csvs/pareto_proxy_rfp/ours_RFP630.csv};
		\end{axis}
		\end{tikzpicture}
}
 \vspace{-1.6em}
 \subcaption{Discovered frontiers}
 \label{fig:frontier}
\end{subfigure}
 \vspace{-1em}
 \caption{Multi-round active learning results and the discovered frontiers on the RFP task under a query limit of $1024$. (a) Midpoint, lower, and upper boundaries show the 50th, 30th, and 70th percentiles, respectively, derived from 10 trials. (b) Colored circles indicates ancestor proteins. }
	\label{fig:RFP}
\end{figure}
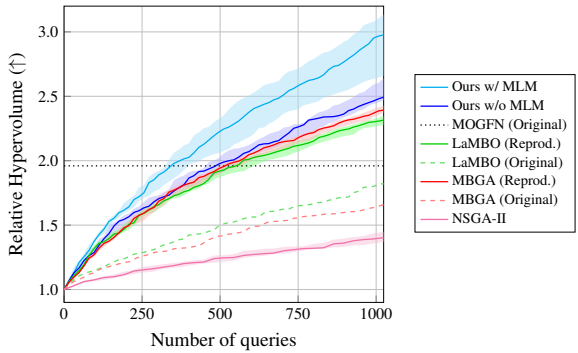
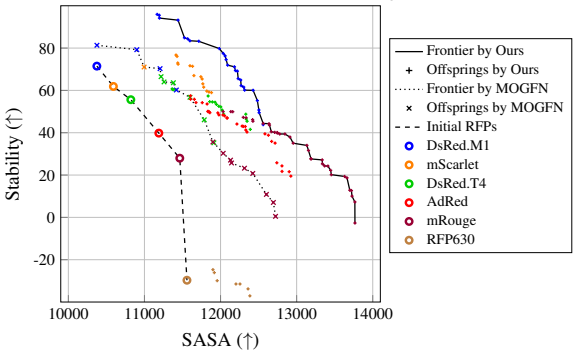

%% file: figures/fig_regex.tex
\begin{figure}[t]
\centering
\resizebox{\columnwidth}{!}{
		\begin{tikzpicture}
		\begin{axis}[
		width=8.5cm,
		height=8.0cm,
		every axis plot/.append style={thick},
		grid=major,
		scaled ticks = false,
		ylabel near ticks,
		tick pos=left,
		tick label style={font=\large},
		xtick={0, 250, 500, 750, 1000},
		xticklabels={0, 250, 500, 750, 1000},
		ytick={1, 50, 100},
		yticklabels={1, 50, 100},
		label style={font=\large},
		xlabel={Number of queries},
		xlabel style={at={(0.5,0.0)}},
		ylabel={Relative Hypervolume ($\uparrow$)},
		ylabel style={align=center, at={(-0.1,0.5)}},
		xmin=0,
		xmax=512,
		ymin=0.8,
		ymax=120.0,
            legend cell align={left},
		legend style={legend columns=1, at={(1.6, 0.77)}, font=\small},
        ]
        \begin{pgfonlayer}{main}
            \addplot[cyan, mark size=1.5pt, no markers] table [x=x_vals, y=y_med, col sep=comma]{csvs/main_regex/Ours_0.0001_use_mlm.csv};
		\addlegendentry{Ours w/ MLM}
            \addplot[blue, mark size=1.5pt, no markers] table [x=x_vals, y=y_med, col sep=comma]{csvs/main_regex/Ours_0.0001.csv};
		\addlegendentry{Ours w/o MLM}
            \addplot[green!80!black!100, mark size=1.5pt, no markers] table [x=x_vals, y=y_med, col sep=comma]{csvs/main_regex/LaMBO_Long.csv};
		\addlegendentry{LaMBO}
    \addplot[red, mark size=1.5pt, no markers] table [x=x_vals, y=y_med, col sep=comma]{csvs/main_regex/MBGA_Long.csv};
		\addlegendentry{MBGA}
  		\addplot[magenta, mark size=1.5pt, no markers] table [x=x_vals, y=y_med, col sep=comma]{csvs/main_regex/NSGA.csv};
		\addlegendentry{NSGA-II}
  \end{pgfonlayer}
            \addplot[name path=r1, blue!20, mark size=1.5pt, no markers, line width=0pt,opacity=0] table [x=x_vals, y=y_ub, col sep=comma]{csvs/main_regex/Ours_0.0001.csv};
            \addplot[name path=r2, blue!20, mark size=1.5pt, no markers, line width=0pt,opacity=0] table [x=x_vals, y=y_lb, col sep=comma]{csvs/main_regex/Ours_0.0001.csv};
          \addplot fill between[ 
            of = r1 and r2, 
            split, 
            every even segment/.style = {blue!50,opacity=0.3},
            ];
            \addplot[name path=b1, green!20, mark size=1.5pt, no markers, line width=0pt,opacity=0] table [x=x_vals, y=y_ub, col sep=comma]{csvs/main_regex/LaMBO_Long.csv};
            \addplot[name path=b2, green!20, mark size=1.5pt, no markers, line width=0pt,opacity=0] table [x=x_vals, y=y_lb, col sep=comma]{csvs/main_regex/LaMBO_Long.csv};
          \addplot fill between[ 
            of = b1 and b2, 
            split, 
            every even segment/.style = {green!50,opacity=0.3},
            ];
    \addplot[name path=g1, red!20, mark size=1.5pt, no markers, line width=0pt,opacity=0] table [x=x_vals, y=y_ub, col sep=comma]{csvs/main_regex/MBGA_Long.csv};
    \addplot[name path=g2, red!20, mark size=1.5pt, no markers, line width=0pt,opacity=0] table [x=x_vals, y=y_lb, col sep=comma]{csvs/main_regex/MBGA_Long.csv};
          \addplot fill between[ 
            of = g1 and g2, 
            split, 
            every even segment/.style = {red!50,opacity=0.3},
            ];
  		\addplot[name path=m1, magenta!20, mark size=1.5pt, no markers, line width=0pt,opacity=0] table [x=x_vals, y=y_ub, col sep=comma]{csvs/main_regex/NSGA.csv};
  		\addplot[name path=m2, magenta!20, mark size=1.5pt, no markers, line width=0pt,opacity=0] table [x=x_vals, y=y_lb, col sep=comma]{csvs/main_regex/NSGA.csv};
              \addplot fill between[ 
            of = m1 and m2, 
            split, 
            every even segment/.style = {magenta!50,opacity=0.3},
            ];
            \addplot[name path=n1, black!20, mark size=1.5pt, no markers, line width=0pt,opacity=0] table [x=x_vals, y=y_ub, col sep=comma]{csvs/main_regex/Ours_0.0001_use_mlm.csv};
                \addplot[name path=n2, black!20, mark size=1.5pt, no markers, line width=0pt,opacity=0] table [x=x_vals, y=y_lb, col sep=comma]{csvs/main_regex/Ours_0.0001_use_mlm.csv};
          \addplot fill between[ 
            of = n1 and n2, 
            split, 
            every even segment/.style = {cyan!60,opacity=0.3},
            ];

		\end{axis}
		\end{tikzpicture}
  }
 \caption{Multi-round active learning results on the 3 Bigrams task when using NEHVI as the batch acquisition function under a query limit of $512$. Midpoint, lower, and upper boundaries show the 50th, 30th, and 70th percentiles, respectively, derived from 10 trials.}
	\label{fig:bigramsapp}
\end{figure}
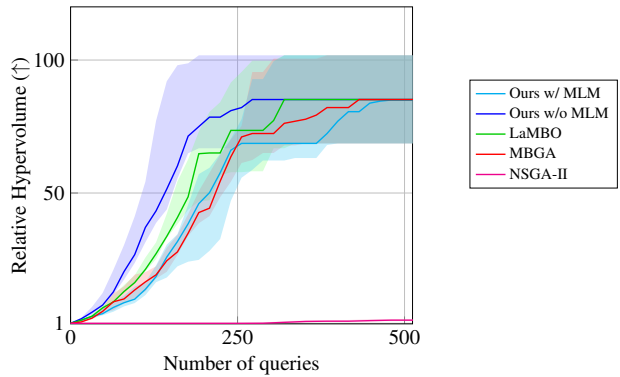

%% file: figures/fig_app_rfp_ucb.tex
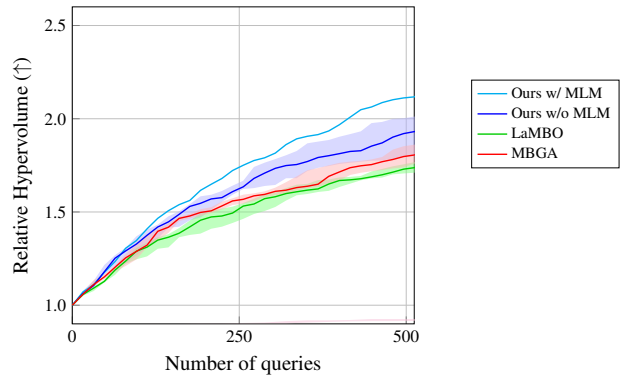
\begin{figure}[t]
\centering
\resizebox{\columnwidth}{!}{
		\begin{tikzpicture}
		\begin{axis}[
		width=8.5cm,
		height=8.0cm,
		every axis plot/.append style={thick},
		grid=major,
		scaled ticks = false,
		ylabel near ticks,
		tick pos=left,
		tick label style={},
		xtick={0, 250, 500, 750, 1000},
		xticklabels={0, 250, 500, 750, 1000},
		ytick={0, 0.5, 1.0, 1.5, 2.0, 2.5, 3.0},
		yticklabels={0, 0.5, 1.0, 1.5, 2.0, 2.5, 3.0},
		label style={},
		xlabel={Number of queries},
		xlabel style={at={(0.5,0.0)}, font=\large},
		ylabel={Relative Hypervolume ($\uparrow$)},
		ylabel style={align=center, at={(-0.1,0.5)}, font=\large},
		xmin=0,
		xmax=512,
		ymin=0.9,
		ymax=2.6,
            legend cell align={left},
		legend style={legend columns=1, at={(1.6, 0.77)}, font=\small},
        ]
        \begin{pgfonlayer}{main}
            \addplot[cyan, mark size=1.5pt, no markers] table [x=x_vals, y=y_med, col sep=comma]{csvs/main_proxy_rfp_ucb/Ours_0.0001_ucb_use_mlm.csv};
		\addlegendentry{Ours w/ MLM}
            \addplot[blue, mark size=1.5pt, no markers] table [x=x_vals, y=y_med, col sep=comma]{csvs/main_proxy_rfp_ucb/Ours_0.0001_ucb.csv};
		\addlegendentry{Ours w/o MLM}
            \addplot[green!80!black!100, mark size=1.5pt, no markers] table [x=x_vals, y=y_med, col sep=comma]{csvs/main_proxy_rfp_ucb/LaMBO_Long_ucb.csv};
		\addlegendentry{LaMBO}
    \addplot[red, mark size=1.5pt, no markers] table [x=x_vals, y=y_med, col sep=comma]{csvs/main_proxy_rfp_ucb/MBGA_Long_ucb.csv};
		\addlegendentry{MBGA}
        \end{pgfonlayer}
            \addplot[name path=n1, black, mark size=1.5pt, no markers, line width=0pt,opacity=0] table [x=x_vals, y=y_ub, col sep=comma]{csvs/main_proxy_rfp_ucb/Ours_0.0001_ucb_use_mlm.csv};
            \addplot[name path=n2, black, mark size=1.5pt, no markers, line width=0pt,opacity=0] table [x=x_vals, y=y_lb, col sep=comma]{csvs/main_proxy_rfp_ucb/Ours_0.0001_ucb_use_mlm.csv};
            \addplot[name path=r1,red!20, mark size=1.5pt, no markers, line width=0pt,opacity=0] table [x=x_vals, y=y_ub, col sep=comma]{csvs/main_proxy_rfp_ucb/Ours_0.0001_ucb.csv};
            \addplot[name path=r2,red!20, mark size=1.5pt, no markers, line width=0pt,opacity=0] table [x=x_vals, y=y_lb, col sep=comma]{csvs/main_proxy_rfp_ucb/Ours_0.0001_ucb.csv};
            \addplot[name path=b1,blue!20, mark size=1.5pt, no markers, line width=0pt,opacity=0] table [x=x_vals, y=y_ub, col sep=comma]{csvs/main_proxy_rfp_ucb/LaMBO_Long_ucb.csv};
            \addplot[name path=b2,blue!20, mark size=1.5pt, no markers, line width=0pt,opacity=0] table [x=x_vals, y=y_lb, col sep=comma]{csvs/main_proxy_rfp_ucb/LaMBO_Long_ucb.csv};
    \addplot[name path=g1,green!20, mark size=1.5pt, no markers, line width=0pt,opacity=0] table [x=x_vals, y=y_ub, col sep=comma]{csvs/main_proxy_rfp_ucb/MBGA_Long_ucb.csv};
    \addplot[name path=g2,green!20, mark size=1.5pt, no markers, line width=0pt,opacity=0] table [x=x_vals, y=y_lb, col sep=comma]{csvs/main_proxy_rfp_ucb/MBGA_Long_ucb.csv};
                  \addplot fill between[ 
            of = r1 and r2, 
            split, 
            every even segment/.style = {blue!30,opacity=0.5},
            ];
                      \addplot fill between[ 
            of = b1 and b2, 
            split, 
            every even segment/.style = {green!50,opacity=0.5},
            ];
                  \addplot fill between[ 
            of = g1 and g2, 
            split, 
            every even segment/.style = {red!30,opacity=0.5},
            ];
                      \addplot fill between[ 
            of = m1 and m2, 
            split, 
            every even segment/.style = {magenta!30,opacity=0.5},
            ];
		\end{axis}
		\end{tikzpicture}
}
 \caption{Multi-round active learning results on the RFP task when using UCBHVI as the batch acquisition function under a query limit of $512$. Midpoint, lower, and upper boundaries show the 50th, 30th, and 70th percentiles, respectively, derived from 10 trials. }
	\label{fig:RFPUCB}
\end{figure}

%% file: figures/fig_div.tex
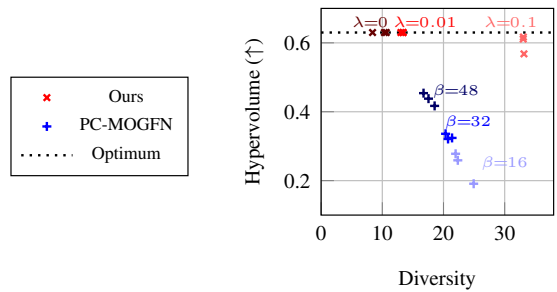
\begin{figure}[t]
\centering
\begin{subfigure}{0.9\columnwidth}
\resizebox{\textwidth}{!}{
\centering
		\begin{tikzpicture}
		\begin{axis}[
		width=4.3cm,
		height=4cm,
		every axis plot/.append style={thick},
		grid=major,
		scaled ticks = false,
		ylabel near ticks,
		tick pos=left,
		tick label style={font=\scriptsize},
		xtick={0, 10, 20, 30},
		xticklabels={0, 10, 20, 30},
		ytick={0, 0.2, 0.4, 0.6},
		yticklabels={0, 0.2 ,0.4, 0.6},
		label style={font=\scriptsize},
		xlabel={Diversity},
		xlabel style={at={(0.5,0.0)}},
		ylabel={Hypervolume ($\uparrow$)},
		ylabel style={align=center, at={(-0.2,0.5)}},
		xmin=0,
		xmax=38,
		ymin=0.1,
		ymax=0.70,
		legend style={legend columns=1, at={(-0.57, 0.67)}, font=\tiny},
        ]
        
            \addplot[red, mark=x, mark size=1.5pt, only marks] coordinates {
                (13.413,0.630)
                (13.034,0.630)
                (13.374,0.630)
            };
		\addlegendentry{Ours}
            \addplot[blue, mark=+, mark size=1.5pt, only marks] coordinates {
                (21.4,0.324)
                (20.744,0.321)
                (20.341,0.336)
            };
            \node at (axis cs: 24, 0.37) {\color{blue}\fontsize{3pt}{2pt}\selectfont $\beta\!\!=\!\!32$};
		\addlegendentry{PC-MOGFN}
            \addplot[black, dotted] coordinates {(0,0.630) (50,0.630)};
		\addlegendentry{Optimum}
            \addplot[red!40!black!100, mark=x, mark size=1.5pt, only marks] coordinates {
                (8.406,0.630)
                (10.647,0.630)
                (10.344,0.630)
            };
            \node at (axis cs: 8, 0.66) {\color{red!40!black!100}\fontsize{3pt}{2pt}\selectfont $\lambda\!\!=\!\!0$};
            \addplot[red!60, mark=x, mark size=1.5pt, only marks] coordinates {
                (33.149,0.568)
                (33.028,0.611)
                (33.054,0.617)
            };
            \node at (axis cs: 17, 0.66) {\color{red}\fontsize{3pt}{2pt}\selectfont $\lambda\!\!=\!\!0.01$};
            \node at (axis cs: 31, 0.66) {\color{red!60}\fontsize{3pt}{2pt}\selectfont $\lambda\!\!=\!\!0.1$};
            \addplot[blue!40, mark=+, mark size=1.5pt, only marks] coordinates {
                (22.34,0.259)
                (24.931,0.191)
                (21.98,0.278)
            };
            \node at (axis cs: 30, 0.25) {\color{blue!40}\fontsize{3pt}{2pt}\selectfont $\beta\!\!=\!\!16$};

            \addplot[blue!40!black!100, mark=+, mark size=1.5pt, only marks] coordinates {
                (18.549,0.417)
                (16.737,0.454)
                (17.549,0.438)
            };
            \node at (axis cs: 22, 0.46) {\color{blue!40!black!100}\fontsize{3pt}{2pt}\selectfont $\beta\!\!=\!\!48$};
		\end{axis}
		\end{tikzpicture}
}
 \vspace{-1.5em}
\end{subfigure}
 \caption{Diversified subset selection results on 2 bigrams task traversing tradeoff parameters. For each tradeoff parameter, $\beta$ for PC-MOGFN, and $\lambda$ for Ours, we plot 3 points for 3 different runs. }
	\label{fig:div}
\end{figure}

%% file: figures/fig_multi.tex
\begin{figure}[t]
\centering
\begin{subfigure}{0.45\textwidth}
	\centering
		\begin{tikzpicture}
		\begin{axis}[
		width=5.5cm,
		height=5.2cm,
		every axis plot/.append style={thick},
		grid=major,
		scaled ticks = false,
		ylabel near ticks,
		tick pos=left,
		tick label style={font=\small},
		xtick={0, 250, 500, 750, 1000},
		xticklabels={0, 250, 500, 750, 1000},
		ytick={0, 0.5, 1.0, 1.5, 2.0, 2.5, 3.0},
		yticklabels={0, 0.5, 1.0, 1.5, 2.0, 2.5, 3.0},
		label style={font=\small},
		xlabel={Number of queries},
		xlabel style={at={(0.5,0.0)}},
		ylabel={Relative Hypervolume ($\uparrow$)},
		ylabel style={align=center, at={(-0.13,0.5)}},
		xmin=0,
		xmax=1024,
		ymin=0.8,
		ymax=2.6,
            legend cell align={left},
		legend style={legend columns=1, at={(0.92, 1.65)}, font=\tiny},
        ]
            \addplot[green!80!black!100, mark size=1.5pt, no markers] table [x=x_vals, y=y_med, col sep=comma]{csvs/github_proxy_rfp/LaMBO_Long.csv};
		\addlegendentry{LaMBO (Corrected, Long)}
		\addplot[green!80!black!100, dashed, mark size=1.5pt, no markers] table [x=x_vals, y=y_med, col sep=comma]{csvs/github_proxy_rfp/LaMBO.csv};
		\addlegendentry{LaMBO (Corrected)}
  		\addplot[green!80!black!100, dotted, mark size=1.5pt, no markers] table [x=x_vals, y=y_med, col sep=comma]{csvs/github_proxy_rfp/LaMBO_Github.csv};
		\addlegendentry{LaMBO (Original)}
  		\addplot[red, mark size=1.5pt, no markers] table [x=x_vals, y=y_med, col sep=comma]{csvs/github_proxy_rfp/MBGA_Long.csv};
		\addlegendentry{MBGA (Corrected, Long)}
		\addplot[red, dashed, mark size=1.5pt, no markers] table [x=x_vals, y=y_med, col sep=comma]{csvs/github_proxy_rfp/MBGA.csv};
		\addlegendentry{MBGA (Corrected)}
		\addplot[red, dotted, mark size=1.5pt, no markers] table [x=x_vals, y=y_med, col sep=comma]{csvs/github_proxy_rfp/MBGA_Github.csv};
		\addlegendentry{MBGA (Original)}
		\end{axis}
		\end{tikzpicture}
 \subcaption{RFP}
\end{subfigure}
\begin{subfigure}{0.45\textwidth}
	\centering
		\begin{tikzpicture}
		\begin{axis}[
		width=5.5cm,
		height=5.2cm,
		every axis plot/.append style={thick},
		grid=major,
		scaled ticks = false,
		ylabel near ticks,
		tick pos=left,
		tick label style={font=\small},
		xtick={0, 250, 500, 750, 1000},
		xticklabels={0, 250, 500, 750, 1000},
		ytick={50, 100},
		yticklabels={50, 100},
		label style={font=\small},
		xlabel={Number of queries},
		xlabel style={at={(0.5,0.0)}},
		ylabel={Relative Hypervolume ($\uparrow$)},
		ylabel style={align=center, at={(-0.13,0.5)}},
		xmin=0,
		xmax=1024,
		ymin=0.8,
		ymax=100.0,
		legend style={legend columns=1, at={(1.57, 0.67)}, font=\tiny},
        ]
            \addplot[green!80!black!100, mark size=1.5pt, no markers] table [x=x_vals, y=y_med, col sep=comma]{csvs/github_regex/LaMBO_Long.csv};
		\addplot[green!80!black!100, dashed, mark size=1.5pt, no markers] table [x=x_vals, y=y_med, col sep=comma]{csvs/github_regex/LaMBO.csv};
  		\addplot[green!80!black!100, dotted, mark size=1.5pt, no markers] table [x=x_vals, y=y_med, col sep=comma]{csvs/github_regex/LaMBO_Github.csv};
  		\addplot[red, mark size=1.5pt, no markers] table [x=x_vals, y=y_med, col sep=comma]{csvs/github_regex/MBGA_Long.csv};
		\addplot[red, dashed, mark size=1.5pt, no markers] table [x=x_vals, y=y_med, col sep=comma]{csvs/github_regex/MBGA.csv};
		\addplot[red, dotted, mark size=1.5pt, no markers] table [x=x_vals, y=y_med, col sep=comma]{csvs/github_regex/MBGA_Github.csv};
		\end{axis}
		\end{tikzpicture}
  \subcaption{3 Bigrams}
\end{subfigure}
 \caption{Multi-round active learning results on the RFP task and 3 bigrams task, comparing performance before and after the implementation corrections. For clarity, only the median performance from 10 trials is depicted. The corrections resulted in significant performance enhancements in these tasks.}
	\label{fig:raise}
\end{figure}
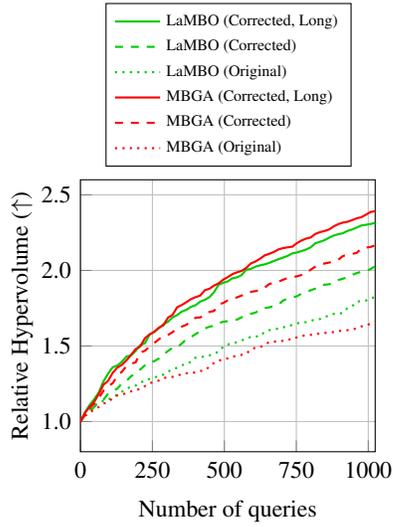
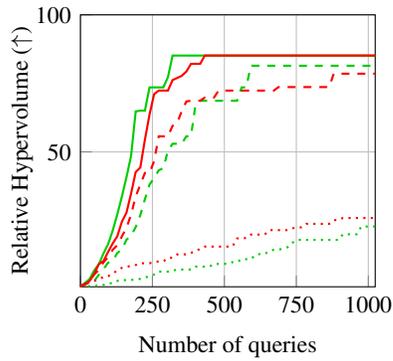